\documentclass{article}

\usepackage[accepted]{icml2015}
\icmltitlerunning{Stay on path: PCA along graph paths}

\usepackage{lipsum}

\usepackage{times}
\usepackage{graphicx} 
\usepackage{subfigure} 
\usepackage{natbib}
\usepackage{algorithm}
\usepackage{algorithmic}
\usepackage{booktabs}

\usepackage[T1]{fontenc}
\usepackage[cmex10]{amsmath}	
\usepackage{multibib}
\newcites{appendix}{Appendix References}
\usepackage{amssymb}
\usepackage{amsthm}
\usepackage{amsfonts}
\usepackage{enumerate}
\usepackage{multirow}
\usepackage{setspace}
\usepackage{array}
\usepackage{graphicx}
\usepackage{float}
\usepackage{flushend}
\usepackage{xfrac}
\usepackage{nicefrac}
\usepackage{paralist}
\usepackage{enumitem}
\usepackage{xspace}
\usepackage{mathtools}  
\usepackage{relsize}
\usepackage{bbm}

\usepackage{scrextend}
\usepackage{tikz}
\usetikzlibrary{arrows}
\usetikzlibrary{shapes}
\usetikzlibrary{shapes.misc}
\usetikzlibrary{fit}
\usetikzlibrary{calc}
\usetikzlibrary{intersections}
\usetikzlibrary{positioning}
\usetikzlibrary{decorations.pathmorphing}
\usetikzlibrary{decorations.pathreplacing}
\newtheorem{theorem}{Theorem}

\newtheorem{lemma}{Lemma}

\newtheorem{assumption}{Assumption}

\newcommand{\netwidth}{\epsilon}

\newcommand{\eqdef}{{\triangleq}}

\DeclareMathOperator*{\argmax}{arg\,max}
\DeclareMathOperator*{\argmin}{arg\,min}

\def\x{\mathbf{x}}
\def\R{\mathbb{R}}

\def\sparsity{k}

\def\numsam{n}

\def\z{\mathbf{z}}
\def\y{\mathbf{y}}
\def\P{\mathcal{P}}
\def\D{\mathcal{D}}

\def\expected{\mathbb{E}}
\def\q{\mathbf{q}}
\def\Q{\mathcal{Q}}
\def\P{\mathcal{P}}

\def\B{\mathcal{B}}

\def\th{\boldsymbol{\theta}}

\def\w{\mathbf{w}}

\newcommand{\dimension}{p}

\newcommand{\aprxrank}{r}
\newcommand{\sphere}{\mathbb{S}^{{\aprxrank-1}}}
\newcommand{\st}{{\mbox{$S$-$T$}}\xspace}
\newcommand{\covariance}{\boldsymbol{\Sigma}}
\newcommand{\empirical}{\widehat{\boldsymbol{\Sigma}}}
\newcommand{\transpose}{{\top}} 
\newcommand{\frob}{{\textnormal{\tiny{F}}}}
\newcommand{\supp}{\text{supp}}
\newcommand{\feasible}{\mathcal{X}(G)}
\newcommand{\proj}{\textnormal{Proj}_{\feasible}}
\newcommand{\lgap}{\lambda_{\rm{gap}}}

\newcommand{\Dnfold}[1]{\D_{\mathsmaller{\dimension}}^{\mathsmaller{(\numsam)}}\mathopen{}\left(#1\right)}

\newtheorem{innercustomlemma}{Lemma}
\newenvironment{customlemma}[1]
  {\renewcommand\theinnercustomlemma{#1}\innercustomlemma}
  {\endinnercustomlemma}

\begin{document} 

\twocolumn[
\icmltitle{Stay on path: PCA along graph paths}

\icmlauthor{\vspace{-5pt}Megasthenis Asteris}{megas@utexas.edu}
\icmlauthor{Anastasios Kyrillidis}{anastasios@utexas.edu}
\icmlauthor{Alexandros G. Dimakis}{dimakis@austin.utexas.edu}
\icmladdress{Department of Electrical and Computer Engineering, The University of Texas at Austin}
\icmlauthor{Han-Gyol Yi}{gyol@utexas.edu}
\icmlauthor{Bharath Chandrasekaran}{bchandra@austin.utexas.edu}
\icmladdress{Department of Communication Sciences \& Disorders, 
The University of Texas at Austin}
\icmlkeywords{Structured Sparse Principal Component Analysis, Sparse PCA, PCA on Graph Paths}

\vskip 0.3in
]

\begin{abstract} 
We introduce a variant of (sparse) PCA
in which the set of feasible support sets is determined by a graph.
In particular, we consider the following setting:
given a directed acyclic graph~$G$ on $p$ vertices corresponding to variables,
the non-zero entries of the extracted principal component must coincide with vertices lying along a path in $G$.

From a statistical perspective, 
information on the underlying network may potentially reduce the number of observations required to recover the population principal component.
We consider the canonical estimator which optimally exploits the prior knowledge by solving a non-convex quadratic maximization on the empirical covariance. 
We introduce a simple network and analyze the estimator under the spiked covariance model.
We show that side information potentially improves
the statistical complexity.

We propose two algorithms to approximate the solution of the constrained quadratic maximization,
and recover a component with the desired properties.
We empirically evaluate our schemes on synthetic and real datasets. 
\end{abstract} 

%
\section{Introduction}\label{sec:introduction}

Principal Component Analysis (PCA) is an invaluable tool in data analysis
and machine learning.
Given a set of ~$n$ centered $p$-dimensional datapoints~${\mathbf{Y} \in \mathbb{R}^{p \times n}}$,
the first principal component is
\begin{align}
 \argmax_{\|\mathbf{x}\|_2 = 1}
 \x^{\transpose}\empirical \x,
 \label{eq:principal-component}
\end{align}
where 
$\empirical = \nicefrac{1}{n} \cdot \mathbf{Y}\mathbf{Y}^{\transpose}$  is the empirical covariance matrix. 
The principal component spans the direction of maximum data variability.
This direction usually involves all $p$ variables of the ambient space, in other words 
the PC vectors are typically non-sparse. However, it is often desirable to obtain a principal component with specific structure, for example limiting the support of non-zero entries. 
From a statistical viewpoint, in the high dimensional regime ${n=O(p)}$, 
the recovery of the true (population) principal
component is only possible if
additional structure information, like sparsity,
is available for the former~\cite{amini2009high,vu2012minimax}.

There are several approaches for extracting a sparse principal component.
Many rely on approximating the solution~to
\begin{align}
	\max_{\x \in \mathbb{R}^{p}}
	\x^{\transpose} \empirical \x,
	\quad \text{subject to} \;
	{\|\x\|_2 = 1}, \; 
	{\|\x\|_0 \leq \sparsity}.
	\label{eq:sPCA}
\end{align} 
The non-convex quadratic optimization is NP hard (by a reduction from maximum clique problem), 
but optimally exploits the side information on the sparsity.

\textbf{Graph Path PCA.}
In this paper we enforce additional structure on the support of principal components. 
Consider a directed acyclic graph (DAG) ${G =(V,E)}$ on $\dimension$ vertices. Let $S$ and $T$ be two additional special vertices and consider all simple paths from $S$ to $T$ on the graph $G$.
Ignoring the order of vertices along a path, let~$\P(G)$ denote the collection of all \st paths in $G$. 
We seek the principal component supported on a path of~$G$,
\textit{i.e.}, the solution to
\begin{align}
   \max_{\x \in \feasible}
   \x^{\transpose} \empirical \x,
   \label{eq:general-problem}
\end{align}
where
\begin{align}
   {\feasible}
   \;\eqdef\;
   \bigl\lbrace
      {\mathbf{x}\! \in\! \mathbb{R}^{\dimension}}:\,
      {\|\mathbf{x}\|_{2}=1}, \;
      {\supp(\x) \in \P(G)}
   \bigr\rbrace.
   \label{feasible-set}
\end{align}

We will argue that this formulation can be used to impose several types of structure on the support of principal components. Note that the covariance matrix $\empirical$ and the graph can be arbitrary: the matrix is capturing data correlations while the graph is a mathematical tool to efficiently describe the possible supports of interest. 
We illustrate this through a few applications.

\textbf{Financial model selection:}
Consider the problem of identifying which companies out of the S{\&}P500 index capture most data variability. 
Running Sparse PCA with a sparsity parameter $k$ will select $k$ companies that maximize explained variance. 
However, it may be useful to enforce more structure: If we must select one company from each business sector (\textit{e.g.}, Energy, Health Care,\textit{etc.}) \textit{how could we identify these representative variables?}

In Section~\ref{sec:experiments}, we show that this additional requirement can be encoded using our graph path framework. We compare our variable selection with Sparse PCA 
and show that it leads to interpretable results.

\textbf{Biological and fMRI networks:}
Several problems involve variables that are naturally connected in a network. 
In these cases our Graph Path PCA can enforce interpretable sparsity structure, especially when the starting and ending points are manually selected by domain experts.  
In section~\ref{sec:experiments}, we apply our algorithm on fMRI data using a graph on regions of interest (ROIs) based on the Harvard-Oxford brain structural atlas \cite{desikan2006automated}.

We emphasize that our applications to brain data is preliminary: the directional graphs we extract are simply based on distance and should not be interpreted as causality, simply as a way of encoding desired supports. 

\textit{What can we say about the tractability of \eqref{eq:general-problem}?}
We note that despite the additional constraints on the sparsity patterns, the number of admissible support sets (\textit{i.e.} \st paths) can be exponential in~$\dimension$, the number of variables.
For example, consider a graph $G$ as follows: $S$ connected to two nodes who are then both connected to two nodes, \textit{etc.} for $\sparsity$ levels and finally connected to $T$. Clearly there are $2^\sparsity$ $\st$ paths and therefore a direct search is not tractable.

{\textbf{Our Contributions:}
\begin{enumerate}
   \item 
   From a statistical viewpoint, 
   we show that side information on the underlying graph~$G$
   can reduce the number of observations required to recover the population principal component~$\x_{\star} \in \feasible$ via~\eqref{eq:general-problem}.
   For our analysis, 
   we introduce a simple, sparsity-inducing network model
   on $p$ vertices partitioned into $k$ layers, with edges from one layer to the next,
   and maximum out-degree~$d$ (Fig.~\ref{fig:layer-graph-ST-included}).
   We show that $\numsam = O\mathopen{}\left(\log\nicefrac{\dimension}{\sparsity} + \sparsity \log{d}\right)$ observations $\y_{i}\sim N(\mathbf{0}, \covariance)$, 
   suffice to obtain an arbitrarily good estimate via~\eqref{eq:general-problem}.
   Our proof follows the steps of~\cite{vu2012minimax}.
   \item
   We complement this with an information-theoretic lower bound on the minimax estimation error, under the spiked covariance model with latent signal $\x_{\star} \in \feasible$,
   which matches the upper bound.
   \item We propose two algorithms 
   for approximating the solution of~\eqref{eq:general-problem},
   based on those of~\cite{yuan2013truncated} and~\cite{papailiopoulos:icml2013,asteris2014nonnegative} for the sparse PCA problem. 
   We empirically evaluate our algorithms on synthetic and real datasets. 
\end{enumerate}

\paragraph{Related Work}
There is a large volume of work on algorithms and the statistical analysis of sparse PCA~\cite{johnstone2004sparse, zou2006sparse, d2008optimal, d2007direct, johnstone2004sparse,vu2012minimax, amini2009high}.
On the contrary, there is limited work 
that considers additional structure on the sparsity patterns.
Motivated by a face recognition application,~\cite{jenatton2010structured}
introduce~\emph{structured sparse PCA} using a regularization that encodes higher-order information about the data. 
The authors design sparsity inducing norms that further promote a pre-specified set of sparsity patterns. 

Finally, we note that the idea of pursuing additional structure on top of sparsity is not limited to PCA:
Model-based compressive sensing seeks sparse solutions under a restricted family of sparsity patterns~\cite{baldassarre2013group, baraniuk2010model, kyrillidis2012combinatorial},
while structure induced by an underlying network is found in~\cite{mairal2011path} for sparse linear regression. 
\section{A Data Model -- Sample Complexity}
\label{sec:data-model}
\begin{figure}[t!]
   \centering
   \includegraphics[width=0.95\linewidth]{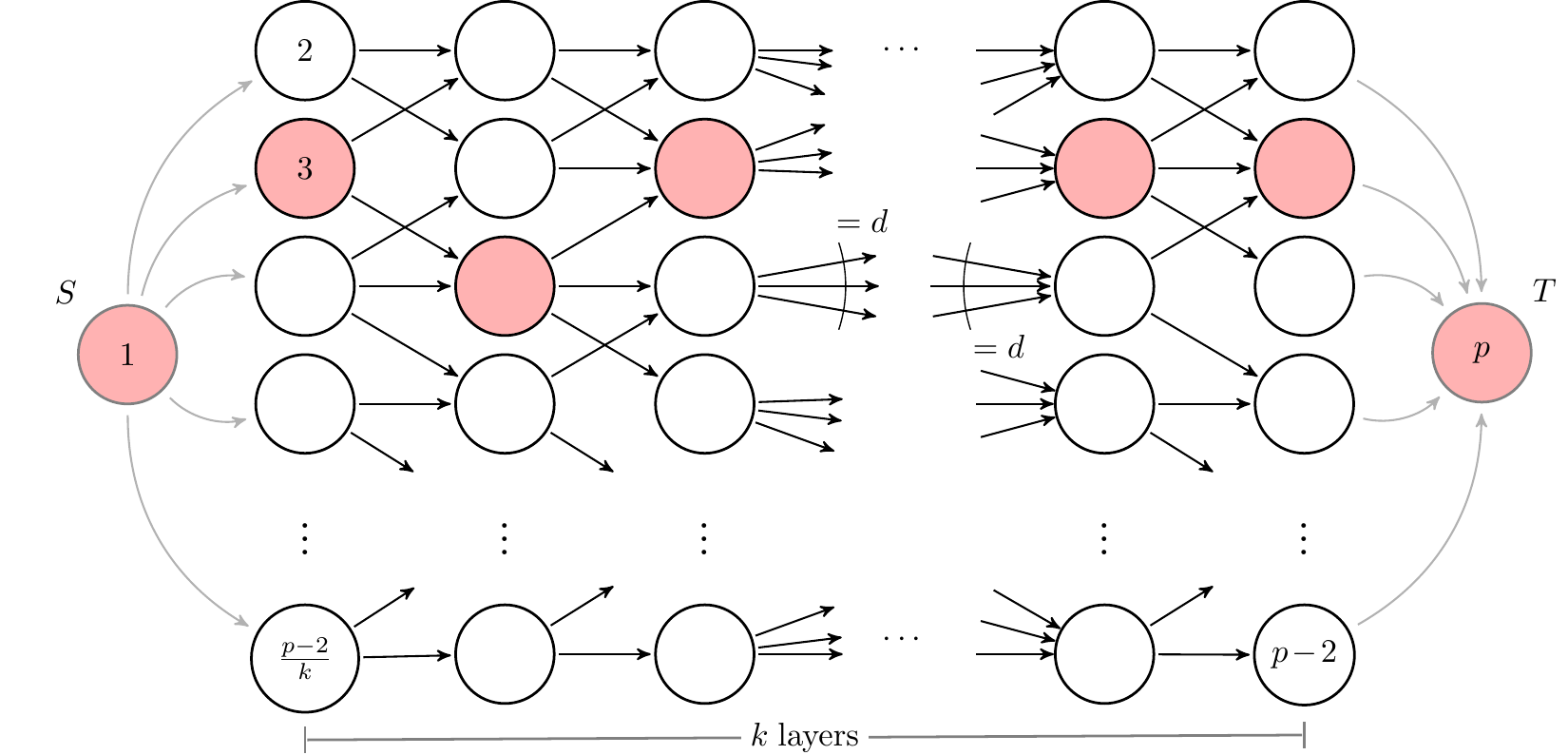}
   \caption{
   	A $(\dimension, \sparsity, d)$-layer graph ${G=(V,E)}$:
	a DAG on ${p}$ vertices,
   	partitioned into~$\sparsity$ disjoint sets (layers) $\mathcal{L}_{1}, \hdots, \mathcal{L}_{\sparsity}$.
	The highlighted vertices form an \st path.
   }
   \label{fig:layer-graph-ST-included}
\end{figure}

\paragraph{The layer graph.}
Consider a directed acyclic graph ${G=(V,E)}$  on~$\dimension$ vertices,
with the following properties:
\begin{enumerate}[label=\textbullet, leftmargin=*, labelindent=0em, align=left, labelsep=.5em, itemindent=0em, labelwidth=1em, itemsep=.2em, topsep=0pt, after=\vspace{.0em}, parsep=0pt]
   \item ${V = \lbrace S, T \rbrace \cup \widehat{V}}$,
   where $S$ is a source vertex, $T$ is a terminal one, 
   and $\widehat{V}$ is the set of remaining~$\dimension-2$ vertices.
   \item $\widehat{V}$ can be partitioned into~$\sparsity$ disjoint subsets (\emph{layers}) $\mathcal{L}_{1}, \hdots, \mathcal{L}_{k}$,
   \textit{i.e.}, $\bigcup_{i} \mathcal{L}_{i} = \widehat{V}$, and ${\mathcal{L}_{i} \cap \mathcal{L}_{j} = \emptyset}$, $\forall i,j \in [k]$, ${i \neq j}$,
   such that:
   \begin{itemize}
    \item[--]
   	 ${\Gamma_{\text{out}}(v) \subset \mathcal{L}_{i+1}}$,
   	 ${\forall v \in \mathcal{L}_{i}}$,
   	 for $i=1,\hdots, {\sparsity-1}$, 
	 where $\Gamma_{\text{out}}(v)$ denotes the out-neighborhood of $v$.
   \item[--] $\Gamma_{\text{out}}(S) = \mathcal{L}_{1}$,
	  and
	  ${\Gamma_{\text{out}}(v) =\lbrace T \rbrace}$, ${\forall v \in \mathcal{L}_{\sparsity}}$. 
   \end{itemize}
\end{enumerate}
In the sequel, for simplicity, 
we will further assume that~${\dimension-2}$ is a multiple of~$\sparsity$
and 
${\rvert \mathcal{L}_{i} \lvert = \sfrac{(\dimension-2)}{\sparsity}}$,~$\forall i \in [k]$.
Further, 
$|\Gamma_{\text{out}}(v)|=d$, $\forall v \in \mathcal{L}_{i}$, $i=1,\hdots,\sparsity-1$,
and 
$|\Gamma_{\text{in}}(v)|=d$, $\forall v \in \mathcal{L}_{i}$, $i=2,\hdots,\sparsity$,
where $\Gamma_{\text{in}}(v)$ denotes the in-neighborhood of~$v$.
In words, the edges from one layer are maximally spread accross the vertices of the next. 
We refer to~$G$ as a \emph{$(\dimension, \sparsity, d)$-layer graph}.

Fig.~\ref{fig:layer-graph-ST-included} illustrates a {$(\dimension, \sparsity, d)$-layer graph}~$G$.
The highlighted vertices form an {\st} path~$\pi$: a \emph{set} of vertices forming a trail from $S$ to $T$.
Let $\P(G)$ denote the collection of \st paths in a graph~$G$ for a given pair of source and terminal vertices.
For the {$(\dimension, \sparsity, d)$-layer graph,
${|\pi| = \sparsity}$, $\forall\pi \in \P(G)$,
and
\begin{align}
   |\P(G)|
   = \lvert \mathcal{L}_{1}\rvert \cdot d^{\sparsity - 1}
   = \tfrac{\dimension-2}{\sparsity} \cdot d^{\sparsity - 1} \leq \textstyle\binom{\dimension-2}{\sparsity},
   \nonumber
\end{align}
since $d \in \lbrace 1, \dots, \nicefrac{(\dimension-2)}{\sparsity}\rbrace$.

\paragraph{Spike along a path.}
We consider the \emph{spiked covariance model},
as in the sparse PCA literature~\cite{johnstone2004sparse, amini2008high
}.
Besides sparsity, we impose additional structure on the latent signal; structure induced by a (known) underlying graph~$G$. 

Consider a $p$-dimensional signal $\x_{\star}$
and a bijective mapping between the~$\dimension$ variables in $\x_{\star}$ and the vertices of $G$.
For simplicity, assume that the vertices of $G$ are labeled so that $x_{i}$ is associated with vertex $i \in V$.
We restrict $\x_{\star}$ in
\begin{align}
   {\feasible}
   \;\eqdef\;
   \bigl\lbrace
      {\mathbf{x}\! \in\! \mathbb{R}^{\dimension}}:\,
      {\|\mathbf{x}\|_{2}=1}, \;
      {\supp(\x) \in \P(G)}
   \bigr\rbrace,
   \nonumber
\end{align}
that is, $\x_{\star}$ is a unit-norm vector 
whose active (nonzero) entries correspond to vertices along a path in $\P(G)$.

We observe $n$ points (samples)
$ \lbrace \y_i \rbrace_{i=1}^{\numsam} \in \R^\dimension$, 
generated randomly and independently as follows:
\begin{align}
   \y_i = \sqrt{\beta} \cdot u_i \cdot \x_{\star} + \z_i,
   \label{gen-observations}
\end{align} 
where the scaling coefficient ${u_i \sim \mathcal{N}(0, 1)}$
and the additive noise ${\z_i \sim \mathcal{N}(\mathbf{0}, \mathbf{I}_{\dimension})}$ 
are independent.
Equivalently,
$\y_{i}$s are \textit{i.i.d.} samples, distributed according to~${\mathcal{N}(\mathbf{0}, \covariance)}$,
where
\begin{align}
 \covariance
 =
 \mathbf{I}_{\dimension} + \beta \cdot \x_{\star} \x_{\star}^{\transpose}.
 \label{pop-covariance}
\end{align}

%
\subsection{Lower bound}

\begin{theorem}[Lower Bound]
   \label{thm:lower-bound-st-included}
   Consider a $(\dimension, \sparsity, d)$-layer graph~$G$
   on $\dimension$ vertices, with ${\sparsity \geq 4}$, and $\log{d} \geq 4H(\sfrac{3}{4})$.
   (Note that ${p-2 \ge k\cdot d}$),
   and a signal $\x_{\star} \in \feasible$.
   Let 
   $\lbrace \y_{i} \rbrace_{i=1}^{\numsam}$
   be a sequence of $\numsam$ random observations,
   independently drawn according to probability density function
   \begin{align*}
   \D_{\dimension}(\x_{\star})
   = 
   {\mathcal{N}\mathopen{}\left(\mathbf{0}, \mathbf{I}_{\dimension} + \beta \cdot \x_{\star} \x_{\star}^{\transpose}\right)},
   \end{align*}
   for some ${\beta > 0}$.
   Let $\Dnfold{\x_{\star}}$ 
   denote the product measure over the~$\numsam$ independent draws.
   Consider the problem of estimating~$\x_{\star}$
   from the $\numsam$ observations, given~$G$.
   There exists ${\x_{\star} \in \feasible}$
   such that for every estimator $\widehat{\x}$,
   \begin{align}
   &\expected_{\scalebox{.65}{$\Dnfold{\x_{\star}}$}}
   \mathopen{}
   \left[\|\widehat{\x} \widehat{\x}^\transpose - \x_{\star}\x_{\star}^\transpose\|_{\frob} \right]
   \geq  \nonumber \\
   &
   \tfrac{1}{2\sqrt{2}} \cdot 
   \sqrt{
   \min\mathopen{} 
   \left\lbrace 
		 1, 
		 ~ \tfrac{C^{\prime} \cdot (1 + \beta)}{\beta^2} 
		 \cdot \tfrac{1}{\numsam}
		 \left( \log\tfrac{p-2}{k} + \tfrac{\sparsity}{4}\log{d} \right)
   \right\rbrace
   }.
   \label{minimax-lb-st-included}
   \end{align}
\end{theorem}
Theorem~\ref{thm:lower-bound-st-included} 
effectively states that
for some latent signal~${\x_{\star} \in \feasible}$,
and observations generated according to the  spiked covariance model,
the minimax 
error is bounded away from zero, unless ${\numsam = \Omega \left(\log{\sfrac{p}{k}}+\sparsity \log d \right)}$.
In the sequel, we provide a sketch proof of Theorem~\ref{thm:lower-bound-st-included}, following the steps of~\cite{vu2012minimax}.

The key idea is to discretize the space $\feasible$
in order to utilize the Generalized Fano Inequality~\cite{yu1997assouad}.
The next lemma summarizes Fano's Inequality for the special case
in which the~$n$ observations are distibuted according to the $\numsam$-fold product measure $\Dnfold{\x_{\star}}$:
\begin{lemma}[Generalized Fano \cite{yu1997assouad}]
\label{lem:Fano_method}
	Let ${\mathcal{X}_{\epsilon} \subset \feasible}$ be a finite set of points
	${\x_1, \ldots, \x_{|\mathcal{X}_{\epsilon}|} \in \feasible}$,
	each yielding a probability measure~$\Dnfold{\x_i}$
	on the $\numsam$ observations.
	If
$ 
   {d(\x_i,\x_j) \ge \alpha},
$ 
for some pseudo-metric\footnote{
	A pseudometric on a set $\mathcal{X}$ is a function 
	${d:\Q^2 \rightarrow \R}$ that satisfies all
	properties of a distance (non-negativity, symmetry, triangle inequality) except the identity of
	indiscernibles:
	${d(\q,\q) = 0}$, ${\forall \q \in \Q}$
	but possibly ${d(\q_1,\q_2) = 0}$ for some ${\q_1 \neq \q_2 \in \Q}$.
} $d(\cdot, \cdot)$ and the Kullback-Leibler divergences satisfy
\begin{align}
	{\rm KL}\mathopen{}\bigl(
	\Dnfold{\x_i} ~\|~ \Dnfold{\x_j} \bigr)
	\;\le\;
	\gamma, \nonumber 
\end{align}
for all ${i \neq j}$,
then for any estimator $\widehat{\x}$
\begin{align} 
	\max_{\x_{i} \in \mathcal{X}_{\epsilon}} 
	\expected_{\Dnfold{\x_{i}}}\mathopen{}\left[ d\mathopen{}\left(\widehat{\x},\x_{i}\right) \right] 
	\;\ge\; 
	\frac{\alpha}{2} \cdot 
	\left(1 - \frac{\gamma + \log 2}{\log |\mathcal{X}_{\epsilon}|}\right).
	\label{eq:gen-fano}
\end{align}
\end{lemma}
Inequality~\eqref{eq:gen-fano}, using the pseudo-metric 
\begin{align*}
	d\left(\widehat{\x}, \x\right)
	\;\eqdef\;
	\|\widehat{\x}\widehat{\x}^{\transpose} - \x \x^{\transpose}\|_{\frob},
\end{align*}
will yield the desired lower bound of Theorem~\ref{thm:lower-bound-st-included} on the minimax estimation error (Eq.~\eqref{minimax-lb-st-included}).
To that end, we need to show the existence of a sufficiently large set $\mathcal{X}_{\epsilon} \subseteq \feasible$ such that
\begin{inparaenum}[\itshape (i)]
 \item the points in~$\mathcal{X}_{\epsilon}$ are well separated under~$d(\cdot, \cdot)$, while
 \item the KL divergence of the induced probability measures is upper appropriately bounded.
\end{inparaenum}
\begin{lemma}
	\label{lem:local-packing:st-included}
	\emph{(Local Packing)}
	Consider a $(\dimension, \sparsity, d)$-layer graph~$G$ on $\dimension$ vertices
	with $k \ge 4$ and $\log{d} \ge 4\cdot H\mathopen{}\left(\sfrac{3}{4}\right)$.
	For any $\epsilon \in (0,1]$,
	there exists a set 
	${\mathcal{X}_{\epsilon} \subset \feasible}$ such that
   \begin{align*}
	  \epsilon / \sqrt{2}
	  \;<\; 
	  \| \x_{i} - \x_{j} \|_{2}
	  \;\le\; 
	  \sqrt{2} \cdot \epsilon,
   \end{align*}
   for all ${\x_i, \x_j \in \mathcal{X}_{\epsilon}}$, ${\x_i \neq \x_j}$,
   and
   \begin{align*}
	  \log |\mathcal{X}_{\epsilon}| 
	  \;\ge\; 
	  \log\frac{p-2}{k} + \sfrac{1}{4}\cdot {\sparsity} \cdot \log{d}.
   \end{align*} 
\end{lemma}
\begin{proof}
   \vspace{-3pt}
   (See Appendix~\ref{sec:proof-local-packing:st-included}).
\end{proof}
For a set $\mathcal{X}_{\epsilon}$ with the properties of 
Lemma~\ref{lem:local-packing:st-included},
taking into account the fact that 
$\|\x_{i}\x_{i}^{\transpose}  - \x_{j}\x_{j}^{\transpose}\|_{\frob}\geq\|\x_{i} - \x_{j}\|_{2}$ (Lemma~{A.1.2} of~\cite{vu2012minimax}), we have
\begin{align}
   d^{2}\mathopen{}\left(\x_{i}, \x_{j}\right)
   =
   \|\x_{i}\x_{i}^{\transpose}  - \x_{j}\x_{j}^{\transpose}\|_{\frob}^2 
   >
   \frac{\epsilon^2}{2} 
   \;{\eqdef}\;
   \alpha^{2}.
   \label{assumption-on-d-st-included}
\end{align}
$\forall{\x_{i},\x_{j} \in \mathcal{X}_{\epsilon}}$, $\x_{i}\neq \x_{j}$.
Moreover, 
{\small
\begin{align}
	&{\rm KL}\mathopen{}\left( \D_\dimension(\x_i) ~\|~ \D_\dimension(\x_j) \right)  
	= \tfrac{\beta}{2(1+\beta)} \cdot
		\Bigl[
				(1+\beta) \times 
		\Bigr. \nonumber \\
	& \qquad\quad\left.
				{\rm Tr}\mathopen{}\left(\bigl(\mathbf{I} - \x_j \x_j^{\transpose}\bigr)\x_i \x_i^{\transpose}\right) 
		\right. 
		\Bigl.- {\rm Tr}\mathopen{}\left(\x_j \x_j^{\transpose}(\mathbf{I} - \x_i \x_i^{\transpose}) \right)\Bigr] 
		\nonumber \\
	& = \tfrac{\beta^2}{4(1 + \beta)} 
		\cdot \|\x_i \x_i^{\transpose} - \x_j \x_j^{\transpose}\|_{\frob}^2
	 \;\le\; 
	 \tfrac{\beta^2}{(1 + \beta)} \cdot \|\x_i - \x_j\|_{2}^2.
	\nonumber
\end{align} 
}
In turn, for the $n$-fold product distribution, and taking into account that
${\|\x_{i} - \x_{j}\|_{2} \le \sqrt{2} \cdot \epsilon}$,
\begin{align}
	{\rm KL}\mathopen{}\bigl( 
	  \Dnfold{\x_i} ~\|~ \Dnfold{\x_j} \bigr)
	\;\le\;
	\frac{2\numsam \beta^2 \epsilon^2}{(1 + \beta)} 
	\;\;{\eqdef}\;\;
	\gamma.
	\label{assumption-on-KL-st-included}
\end{align}
Eq.~\eqref{assumption-on-d-st-included} and~\eqref{assumption-on-KL-st-included} establish the parameters $\alpha$ and $\gamma$ required by Lemma~\ref{lem:Fano_method}.
Substituting those into~\eqref{eq:gen-fano},
along with the lower bound of Lemma~\ref{lem:local-packing:st-included} on $|\mathcal{X}_{\epsilon}|$,
we obtain
{\small
\begin{align}
    \max_{\x_{i} \in \mathcal{X}_{\epsilon}} 
	\expected_{\scalebox{.65}{$\Dnfold{\x_{i}}$}}
	  \mathopen{}
	  \left[ d\mathopen{}\left(\widehat{\x},\x_{i}\right) \right] 
	\ge \frac{\epsilon}{2\sqrt{2}} \mathopen{}
		 \left[1 - \frac{\numsam \frac{2\epsilon^2\beta^2}{(1 + \beta)} + \log 2}{\log |\mathcal{X}_{\epsilon}|}\right]\mathclose{}. 
	\label{close-to-lower-bound-st-included}
\end{align}
}%
The final step towards establishing the desired lower bound in~\eqref{minimax-lb-st-included} is to appropriately choose~$\epsilon$.
One can verify that if
\begin{align}
      \epsilon^2 
      = 
      \min\mathopen{} 
         \left\lbrace 
      		 1, 
      		 ~\tfrac{C^{\prime} \cdot (1 + \beta)}{\beta^2} 
      		 \cdot \tfrac{1}{\numsam}
      		 \left(
      		 \log\tfrac{p-2}{k} + \tfrac{\sparsity}{4}\cdot \log{d}
      		 \right)
         \right\rbrace,
      \label{epsilon-range-st-included}
\end{align}
where ${C^{\prime}>0}$ is a constant to be determined, then
\begin{align}
	\numsam 
	\cdot \frac{2 \epsilon^2 \beta^2 }{(1+\beta)}
	\frac{1}{\log |\mathcal{X}_{\epsilon}|} \leq \frac{1}{4} 
	\;\;\text{ and }\;\;
	\log |\mathcal{X}_{\epsilon}| \geq 4 \log 2,
	\label{conditions-for-constant-eps-st-included}
\end{align}
(see Appendix~\ref{sec:lb-proof-details} for details).
Under the conditions in~\eqref{conditions-for-constant-eps-st-included}, 
the inequality in~\eqref{close-to-lower-bound-st-included} implies that
\begin{align}
    \max_{\x_{i} \in \mathcal{X}_{\epsilon}} 
	\expected_{
			\scalebox{.65}{$\Dnfold{\x_{i}}$}
		 }
	  \mathopen{}
	  \left[ d\mathopen{}\left(\widehat{\x},\x_{i}\right) \right] 
	\ge \tfrac{1}{2\sqrt{2}} \cdot \epsilon.
\end{align}
Substituting $\epsilon$ according to~\eqref{epsilon-range-st-included}, 
yields the desired result in~\eqref{minimax-lb-st-included}, completing the proof of Theorem~\ref{thm:lower-bound-st-included}.

%
\subsection{Upper bound}
Our upper bound is based on the estimator obtained via the constrained quadratic maximizaton in~\eqref{eq:general-problem}.
We note that the analysis is not restricted to the spiked covariance model;
it applies to a broader class of distributions (see Assum.~\ref{ub-assumption}).
\begin{theorem}[Upper bound]
\label{thm:ub}
Consider a $(p,k,d)$-layer graph~$G$
and ${\mathbf{x}_{\star} \in \feasible}$.
Let 
$\lbrace \y_{i} \rbrace_{i=1}^{\numsam}$
be a sequence of~$\numsam$ \textit{i.i.d.}
 $\mathcal{N}\mathopen{}\left(\mathbf{0}, \boldsymbol{\Sigma}\right)$ samples,
 where $\boldsymbol{\Sigma} \succeq \mathbf{0} $
 with eigenvalues $\lambda_{1} > \lambda_{2} \ge \hdots$,
 and principal eigenvector $\x_{\star}$.
Let~$\empirical$ be the empirical covariance of the~$\numsam$ samples,
$\widehat{\x}$ the estimate of~$\x_{\star}$ obtained via~\eqref{eq:general-problem},
and 
$
	\epsilon
	\;\eqdef\;
	\|\widehat{\x} \widehat{\x}^{\transpose} - \x_{\star} \x_{\star}^{\transpose}\|_{\frob}.
	\nonumber
$
Then, 
\begin{align*}
	\expected\mathopen{}
	\left[
		\epsilon
	\right] 
	\le 
	C \cdot \frac{\lambda_1}{\lambda_{1}-\lambda_{2}}
	\cdot \frac{1}{n} 
	\cdot \max\mathopen{}
	\left\lbrace
		\sqrt{n A}, \, A  
	\right\rbrace,
\end{align*}
where 
${
	{A = O\mathopen{}\left({\log\tfrac{\dimension-2}{\sparsity} + \sparsity \log{d}}\right)}
}$.
\end{theorem}
In the sequel, we provide a sketch proof of Theorem~\ref{thm:ub}.
The proof closely follows the steps of~\cite{vu2012minimax} in developing their upper bound for the the sparse PCA problem.
\begin{lemma}[Lemma~$3.2.1$~\cite{vu2012minimax}]
	\label{lem:upper_01}
	Consider ${\boldsymbol{\Sigma} \in \mathbb{S}_{+}^{\dimension \times \dimension}}$,
	 with principal eigenvector $\x_{\star}$ and
	$\lgap \eqdef \lambda_1 - \lambda_2(\boldsymbol{\Sigma})$.
	For any~$\widetilde{\x} \in \R^\dimension$ with ${\|\widetilde{\x}\|_2 = 1}$,
\begin{align}
	\tfrac{\lambda_{1}-\lambda_{2}}{2}
	\cdot \|\widetilde{\x} \widetilde{\x}^{\transpose} - \x_{\star} \x_{\star}^{\transpose}\|_{\frob}^{2}
	\;\leq\;
	\left\langle 
		\boldsymbol{\Sigma}, \,
		{
			\x_{\star} \x_{\star}^{\transpose}
			-
			\widetilde{\x} \widetilde{\x}^{\transpose}
		}
	\right\rangle.
	\nonumber
\end{align} 
\end{lemma}
Let $\widehat{\x}$ be an estimate of $\x_{\star}$ via~\eqref{eq:general-problem},
and
$\epsilon \eqdef
   \|\widehat{\x} \widehat{\x}^{\transpose} - \x_{\star} \x_{\star}^{\transpose}\|_{\frob}$.
From Lemma~\ref{lem:upper_01}, it follows (see~\cite{vu2012minimax}) that
\begin{align}
   \tfrac{\lambda_{1}-\lambda_{2}}{2} \cdot \epsilon^{2}
	\leq
	\bigl\langle 
		{\empirical-\covariance}, \,
		{\widehat{\x} \widehat{\x}^{\transpose} - \x_{\star} \x_{\star}^{\transpose}}
	\bigr\rangle.
	\label{upbound-on-eps-square}
\end{align}
Both $\x_{\star}$ and $\widehat{\x}$ belong to~$\feasible$:
unit-norm vectors, with support of cardinality $\sparsity+2$ coinciding with a path in~$\P(G)$.
Their difference is supported in $\P^{2}(G)$: the collection of sets formed by the union of two sets in~$\P(G)$.
Let $\mathcal{X}^{2}(G)$ denote the set of unit norm vectors supported in~$\P^{2}(G)$.
Via an appropriate upper bounding of the right-hand side of~\eqref{upbound-on-eps-square},
\cite{vu2012minimax} show that 
\begin{align}
	\expected\mathopen{}\left[\epsilon\right] 
	\le
	\tfrac{\widehat{C}}{{\lambda_{1}-\lambda_{2}}} \cdot 
	\expected\mathopen{}
	\left[
	  \textstyle
	  \sup_{\substack{\th \in \mathcal{X}^{2}}}
	  \bigl| 
	  \th^{\transpose} 
		 \bigl(\empirical - \boldsymbol{\Sigma}\bigr) 
	  \th
	  \bigr|
	\right],
	\nonumber
\end{align}
for an appropriate constant ${\widehat{C}>0}$.
Further, under the assumptions on the data distribution,
and utilizing a result due to~\cite{mendelson2010empirical},
\begin{align}
   \expected&\mathopen{}
   \left[ 
   	\sup_{\substack{\th \in \mathcal{X}^{2}}}
   	\bigl|
	\th^{\transpose} 
		\bigl(\empirical - \boldsymbol{\Sigma}\bigr) 
	\th
	\bigr|
   \right] 
   \leq 
   C^{\prime}  
   K^2  
    \frac{\lambda_1}{n}
    \max\mathopen{}
   \left\{ \sqrt{n}A, \,A^{2}\right\},
   \nonumber
\end{align} 
for $C^{\prime}$ and $K$ constants depending on the distribution, and
\begin{align}
   A
   \;\eqdef\;
	\expected_{
	  {\mathbf{Y}\sim N(\mathbf{0}, \mathbf{I}_{\dimension})}
	}\mathopen{}
	\left[\textstyle
		\sup_{
			\substack{\th \in \mathcal{X}^{2}}}
			\langle \mathbf{Y}, \th \rangle 
	\right].
   \label{eq:Asupremum}
\end{align}
This reduces the problem of bounding $\expected\mathopen{}\left[\epsilon\right]$
to bounding the supremum of a Gaussian process.
Let $\mathcal{N}_{\delta} \subset \mathcal{X}^{2}(G)$ be a minimal $\delta$-covering of $\mathcal{X}^{2}(G)$ in the Euclidean metric
with the property that $\forall \x \in  \mathcal{X}^{2}(G)$,  $\exists \y \in \mathcal{N}_{\delta}$ such that $\|\x - \y\|_2 \leq \delta$ and $\supp(\x - \y) \in \mathcal{P}^{2}(G)$. 
Then,
\begin{align}
	\textstyle
	\sup_{
		\substack{\th \in \mathcal{X}^{2}}}
		 \langle \mathbf{Y}, \th \rangle 
	\;\leq\;
	\displaystyle
	(1 - \delta)^{-1}
	\cdot \max_{\th \in \mathcal{N}_{\delta}} \langle \mathbf{Y}, \th \rangle.
	\label{boundusingD}
\end{align} 
Taking expectation w.r.t. $\mathbf{Y}$ and applying a union bound on the right hand side, we conclude
\begin{align}
	A
	\leq 
	\widetilde{C} \cdot 
	{(1 - \delta)}^{-1} \cdot \sqrt{\log |\mathcal{N}_{\delta}|}.
	\label{A-ub}
\end{align} 
It remains to construct a $\delta$-covering $\mathcal{N}_{\delta}$ with the 
desired properties.
To this end, we associate isometric copies of $\mathbb{S}_2^{2\sparsity+1}$ with each support set in~$\P^2(G)$.
It is known that there exists a minimal $\delta$-covering for $\mathbb{S}_2^{2\sparsity+1}$ with cardinality at most~$(1 + \nicefrac{2}{\delta})^{2\sparsity+2}$. 
The union of the local $\delta$-nets forms a set $\mathcal{N}_{\delta}$ with the desired properties.
Then,
\begin{align}
	\log {|\mathcal{N}_{\delta}|}
	&\leq\;
	\log{|\P^{2}(G)|} + 2(\sparsity+1) \log (1 + \nicefrac{2}{\delta}) \nonumber\\
	&=\;
	O\mathopen{}\left(\log \tfrac{\dimension-2}{\sparsity} + \sparsity \log d\right),
	\nonumber
\end{align}
for any constant $\delta$.
Substituting in~\eqref{A-ub}, implies the desired bound on $\expected\mathopen{}\left[\epsilon\right]$,
completing the proof of Theorem~\ref{thm:ub}.

%
\section{Algorithmic approaches}{\label{sec:proposed}}
We propose two algorithms for
approximating the solution of the 
constrained quadratic maximization in~\eqref{eq:general-problem}:
\begin{enumerate}[leftmargin=*, labelindent=0em, align=left, labelsep=.5em, itemindent=0em, labelwidth=1em, itemsep=.2em, nosep, 
]
   \item The first is an adaptation of the \emph{truncated power iteration} method of~\cite{yuan2013truncated} for the problem of computing sparse eigenvectors.
   \item The second relies on approximately solving~\eqref{eq:general-problem} on a low rank approximation of~$\empirical$, similar to~\cite{papailiopoulos:icml2013, asteris2014nonnegative}. 
\end{enumerate}
Both algorithms rely on a projection operation
from~$\R^{\dimension}$ onto the feasible set~$\feasible$, for a given graph~${G=(V,E)}$.
Besides the projection step,
the algorithms are oblivious to the specifics of the constraint set,\footnote{%
	For Alg.~\ref{algo:eps-net}, the observation holds under mild assumptions:
	$\feasible$ must be such that $\|\x\|_{2}=\Theta(1)$, while $\pm\x \in \feasible$ 
	should both achieve the same objective value.}
and can adapt to different constraints by modifying the projection operation. 


\subsection{Graph-Truncated Power Method}
\label{sec:graph-truncated-power-method}
\begin{algorithm}[bthp!]
   \small
   \caption{Graph-Truncated Power Method}
   \label{algo:graphPM}
   \begin{algorithmic}[1] 
      \INPUT $\empirical \in \R^{\dimension \times \dimension},\; {G=(V,E)}, \; {\x_0 \in \R^\dimension}$
      \STATE $i \gets 0$
      \REPEAT
	 \STATE $\w_{i} \gets \empirical \x_{i}$
	 \STATE $\mathbf{x}_{i+1} \gets \proj(\mathbf{w}_{i})$
	 \STATE $i \gets i+1$
      \UNTIL{Convergence/Stop Criterion}
      \OUTPUT $\mathbf{x}_{i}$
   \end{algorithmic}
\end{algorithm}
We consider a simple iterative procedure,
similar to the {truncated power method} of~\cite{yuan2013truncated} for the problem of computing sparse eigenvectors.
Our algorithm produces sequence of vectors $\x_{i} \in \feasible$, $i \ge 0$,
that serve as intermediate estimates of the desired solution of~\eqref{eq:general-problem}.

The procedure is summarized in Algorithm~\ref{algo:graphPM}.
In the $i$th iteration, 
the current estimate ${\x_{i}}$ is multiplied by the empirical covariance $\empirical$,
The product ${\w_{i} \in \R^{\dimension}}$ is projected back to the feasible set $\feasible$, yielding the next estimate $\x_{i+1}$.
The core of Algorithm~\ref{algo:graphPM} lies in the projection operation, 
\begin{align}
   \proj(\mathbf{w})
   \;\eqdef\;
   \argmin_{
      \x \in \feasible
	    } 
   \frac{1}{2}\|\x - \w\|_{2}^{2},
   \label{projector}
\end{align}
which is analyzed separately in Section~\ref{sec:projection-step}.
The initial estimate $\x_0$ can be selected randomly or based on simple heuristics, \textit{e.g.}, the projection on~$\feasible$ of the column of~$\empirical$ corresponding to the largest diagonal entry.
The algorithm terminates when some convergence criterion is satisfied.


The computational complexity (per iteration) of Algorithm~\ref{algo:graphPM}
is dominated by the cost of matrix-vector multiplication 
and the projection step.
The former is~$O(\sparsity \cdot \dimension)$,
where~$\sparsity$ is cardinality of the largest support in $\feasible$.
The projection operation for the particular set~$\feasible$, boils down to solving the longest path problem on a weighted variant of the DAG $G$ (see Section~\ref{sec:projection-step}), which can be solved in time 
${O(|V|+|E|)}$, \textit{i.e.}, linear in the size of $G$.

\subsection{Low-Dimensional Sample and Project}

The second algorithm outputs an estimate of the desired solution of~\eqref{eq:general-problem}
by (approximately) solving the constrained quadratic maximization 
\emph{not} on the original matrix~$\empirical$, but on a low rank approximation~$\empirical_{\aprxrank}$ of $\empirical$, instead:
\begin{align}
   \empirical_{\aprxrank}
   =\sum_{i=1}^\aprxrank \lambda_{i}\mathbf{q}_{i} \mathbf{q}_{i}^{\transpose}
   =\sum_{i=1}^\aprxrank \mathbf{v}_{i} \mathbf{v}_{i}^{\transpose}
   =
   \mathbf{V}\mathbf{V}^{\transpose},
   \label{eq:Ad-as-VVt}
\end{align}
where~$\lambda_{i}$ is the $i$th largest eigenvalue of~$\empirical$,
${\mathbf{q}}_i$ is the corresponding eigenvector,
${\mathbf{v}_{i} \eqdef \sqrt{\lambda_i} \cdot {\mathbf{q}}_i}$,
and $\mathbf{V}$ is the $\dimension \times \aprxrank$ matrix whose $i$th column is equal to $\mathbf{v}_{i}$.
The approximation rank~$\aprxrank$ is an accuracy parameter;
typically, $\aprxrank \ll \dimension$.

Our algorithm operates\footnote{
   Under the spiked covariance model,
   this approach may be asymptotically unsuitable; 
   as the ambient dimension increases, 
   it with fail to recover the latent signal. 
   Empirically, however, if the spectral decay of $\empirical$ is sharp,
   it yields very competitive results. 
}
on $\empirical_{\aprxrank}$
and seeks
\begin{align}
   \x_{\aprxrank} 
   \; \eqdef \;
      \argmax_{
      \x \in \feasible
      }\;
   \mathbf{x}^{\transpose} \, \empirical_{\aprxrank} \, \mathbf{x}.
   \label{eq:low-rank-problem}
\end{align}
The motivation is that an (approximate) solution for the low-rank problem in~\eqref{eq:low-rank-problem}
can be efficiently computed. 
Intuitively, if $\empirical_{\aprxrank}$ is a sufficiently good approximation of the original matrix $\empirical$, 
then $\mathbf{x}_{\aprxrank}$ would perform similarly to the solution $\mathbf{x}_{\star}$ of the original problem~\eqref{eq:general-problem}.

\begin{algorithm}[tb!]
   \small
   \caption{Low-Dimensional Sample and Project}
   \label{algo:eps-net}
   \begin{algorithmic}[1] 
   \INPUT $\empirical \in \R^{\dimension \times \dimension}$, ${G=(V,E)}$, $\aprxrank \in [\dimension]$, $\netwidth > 0$
   \STATE $[\mathbf{Q},\boldsymbol{\Lambda}]\gets \texttt{svd}(\empirical,\aprxrank)$
   \STATE $\mathbf{V} \gets \mathbf{Q}\boldsymbol{\Lambda}^{\sfrac{1}{2}}$
   \hfill\COMMENT{${\empirical}_{\aprxrank} \eqdef \mathbf{V}\mathbf{V}^{\transpose}$}
   \STATE $\mathcal{C}\gets \emptyset$
   \hfill \COMMENT{Candidate solutions}
   \FOR{ $i=1:O(\netwidth^{-\aprxrank}\cdot \log{\dimension})$}
      \STATE $ \mathbf{c}_{i} \gets $ uniformly sampled from $\sphere$
      \STATE $\w_{i} \gets \mathbf{V}\mathbf{c}_{i}$
      \STATE $\x_{i} \gets \proj(\w_{i})$
      \STATE $\mathcal{C}= \mathcal{C} \cup \left\lbrace \mathbf{x}_{i} \right\rbrace$
   \ENDFOR
   \OUTPUT $\widehat{\mathbf{x}}_{\aprxrank} \gets \argmax_{\mathbf{x} \in \mathcal{C}} \|\mathbf{V}^{\transpose}\mathbf{x}\|_2^2$
   \end{algorithmic}
\end{algorithm}

\textbf{The Algorithm. }
Our algorithm samples points from the low-dimensional principal subspace of $\empirical$,
and projects them on the feasible set $\feasible$, 
producing a set of candidate estimates for $\x_{\aprxrank}$.
It outputs the candidate that maximizes the objective in~\eqref{eq:low-rank-problem}.
The exact steps are formally presented in Algorithm~\ref{algo:eps-net}.
The following paragraphs delve into the details of Algorithm~\ref{algo:eps-net}.

\subsubsection{The Low Rank Problem}
The rank-$\aprxrank$ maximization in~\eqref{eq:low-rank-problem}
can be written as
\begin{align}
	\max_{\mathbf{x} \in {\feasible}}\mathbf{x}^{\transpose}\empirical_{\aprxrank}\mathbf{x}
	=
	\max_{\mathbf{x} \in {\feasible}}
		  \| \mathbf{V}^{\transpose}\mathbf{x} \|_2^2,
	\label{eq:low-rank-original-problem}
\end{align}
and in turn (see \cite{asteris2014nonnegative} for details),
as a double maximization over the variables~${\mathbf{c} \in \sphere}$ and~${\mathbf{x}\in \mathbb{R}^{\dimension}}$:
\begin{align}
	\max_{\mathbf{x} \in {\feasible}}
		  \| \mathbf{V}^{\transpose}\mathbf{x} \|_2^2
	&=
	\max_{
		\mathbf{c} \in \sphere
	} 
	\max_{
			\mathbf{x} \in {\feasible}
	} 
	{\bigl(  \left(\mathbf{V}\mathbf{c} \right)^{\transpose}\mathbf{x} \bigr)}^2.
	\label{eq:rank-r-double-optimization}
\end{align}

\textbf{The rank-$1$ case.\;}
Let ${\mathbf{w} \eqdef \mathbf{V}\mathbf{c}}$;
$\mathbf{w}$ is only a vector in~$\mathbb{R}^{\dimension}$.
For given $\mathbf{c}$ and $\w$,
the $\mathbf{x}$ that maximizes the objective in \eqref{eq:rank-r-double-optimization} (as a function of $\mathbf{c}$) is
\begin{align}
	\mathbf{x}(\mathbf{c})
	\in
	\argmax_{\mathbf{x} \in {\feasible}}  \left( \mathbf{w}^{\transpose}\mathbf{x} \right)^2.
	\label{eq:rank1-original-problem}
\end{align} 
The maximization in~\eqref{eq:rank1-original-problem} is nothing but a \mbox{rank-$1$} instance of the
maximization in~\eqref{eq:low-rank-original-problem}.
Observe that if $\x \in \feasible$,
then $-\x \in \feasible$, 
and the two vectors attain the same objective value. 
Hence,~\eqref{eq:rank1-original-problem} can be simplified: 
\begin{align}
   \x(\mathbf{c})
   \in
   \argmax_{\mathbf{x} \in {{\mathcal{X}(G)}}}  
   \mathbf{w}^{\transpose}\mathbf{x}.
   \label{eq:remove-the-square}
\end{align}
Further, since ${\|\x\|_{2}=1}$, ${\forall \x \in \feasible}$,
the maximization in~\eqref{eq:remove-the-square}
is equivalent to minimizing $\frac{1}{2}\|\w-\x\|_{2}^2$.
In other words,
$\x(\mathbf{c})$ is just the projection of $\w \in \R^{\dimension}$ onto ${\mathcal{X}(G)}$:
\begin{align}
   \x(\mathbf{c})
   \;\in\;
   \proj(\mathbf{w}). 
   \label{xcasprojection}
\end{align}
The projection operator is described in Section~\ref{sec:projection-step}.

\textbf{Multiple rank-$1$ instances.\;}
Let $\bigl(\mathbf{c}_{\aprxrank}, {\mathbf{x}_{\aprxrank}}\bigr)$ 
denote a pair that attains the maximum value in~\eqref{eq:rank-r-double-optimization}.
If $\mathbf{c}_{\aprxrank}$ was known, 
then $\x_{r}$ would coincide with
the projection $\x(\mathbf{c}_{\aprxrank})$
of ${\w=\mathbf{V}\mathbf{c}_{\aprxrank}}$ on the feasible set,
according to~\eqref{xcasprojection}.

Of course, the optimal value~$\mathbf{c}_{\aprxrank}$ of the auxiliary variable is {not} known. 
Recall, however, that $\mathbf{c}_{\aprxrank}$ lies on the low dimensional manifold $\sphere$.
Consider an $\epsilon$-net $\mathcal{N}_{\epsilon}$ covering the $\aprxrank$-dimensional unit sphere $\sphere$;
Algorithm~\ref{algo:eps-net} constructs such a net by random sampling. 
By definition, 
$\mathcal{N}_{\epsilon}$ contains at least one point,
call it $\widehat{\mathbf{c}}_{\aprxrank}$,
in the vicinity of $\mathbf{c}_{\aprxrank}$.
It can be shown that the corresponding solution $\x(\widehat{\mathbf{c}}_{\aprxrank})$ 
in~\eqref{eq:rank1-original-problem}  will perform approximately as well as the optimal solution~$\mathbf{x}_{\aprxrank}$, in terms of the quadratic objective in \eqref{eq:rank-r-double-optimization}, for a large, but tractable, number of points in the $\epsilon$-net of $\sphere$.

\subsection{The Projection Operator}
\label{sec:projection-step}
Algorithms~\ref{algo:graphPM}, and~\ref{algo:eps-net}
rely on a projection operation 
from $\R^{\dimension}$ onto the feasible set $\feasible$ (Eq.~\eqref{feasible-set}).
We show that the projection effectively
reduces to solving the \emph{longest path problem} on (a weighted variant of)~$G$.

The projection operation, defined in Eq.~\eqref{projector},
can be equivalently\footnote{
   It follows from expanding the quadratic $\frac{1}{2}\|\x-\w\|^{2}_{2}$ and the fact that $\|\x\|_{2}=1$, $\forall \x \in \feasible$.
}
written as
\begin{align*}
   \proj(\mathbf{w})
   \;{\eqdef}\;
   \argmax_{
      \x \in \feasible
   }
   \w^{\transpose} \x.
\end{align*}
For any~${\x \in \feasible}$, $\supp(\x) \in \P(G)$.
For a given set $\pi$, by the Cauchy-Schwarz inequality,
\begin{align}
	\mathbf{w}^{\transpose}\mathbf{x}
	=
	{\textstyle \sum_{i \in \pi} w_{i}x_{i}}
	\;\le\;
	{\textstyle\sum_{i \in \pi} w_{i}^{2} }
	=
	\widehat{\mathbf{w}}^{\transpose} \mathbf{1}_{\pi},
	\label{ub-on-objective-real}
\end{align}
where ${\widehat{\mathbf{w}} \in \R^{\dimension}}$
is the vector obtained by squaring the entries of $\w$,
\textit{i.e.}, ${\widehat{w}_{i} = w_{i}^{2}}$, ${\forall i \in [n]}$,
and ${\mathbf{1}_{\pi} \in \{0,1\}^{\dimension}}$ denotes the characteristic of~$\pi$.
Letting $\x[\pi]$ denote the subvector of $\x$ supported on~$\pi$,
equality in~\eqref{ub-on-objective-real}
can be achieved by $\x$ such that $\mathbf{x}[{\pi}] = \mathbf{w}[{\pi}]/ \|\mathbf{w}[{\pi}]\|_2$,
and $\x[\pi^{c}] = \mathbf{0}$.

Hence, the problem in~\eqref{ub-on-objective-real} reduces to determining 
\begin{align}
	\pi({\mathbf{w}})
	\in \argmax_{ \pi \in \P(G) } \widehat{\mathbf{w}}^{\transpose}\mathbf{1}_{\pi}.
	\label{optimal-path-binary}
\end{align}
Consider a weighted graph $G_{\mathbf{w}}$,
obtained from $G=(V,E)$ by assigning weight ${\widehat{w}_{v} = w_{v}^{2}}$
on vertex $v \in V$.
The objective function in \eqref{optimal-path-binary} equals the \emph{weight of the path} $\pi$ in $G_{\mathbf{w}}$,
\textit{i.e.}, the sum of weights of the vertices along $\pi$.
Determining the optimal support $\pi(\w)$ for a given $\w$, 
is equivalent to solving the \emph{longest (weighted) path problem}\footnote{
   The longest path problem is commonly defined on graphs with weighted edges instead of vertices. 
   The latter is trivially transformed to the former:
   set ${w(u,v)\leftarrow w(v)}$, ${\forall (u,v) \in E}$,
   where $w(u,v)$ denotes the weight of edge $(u,v)$,
   and $w(v)$ that of vertex $v$.
   Auxiliary edges can be introduced for source vertices.
}
on $G_{\mathbf{w}}$.

The longest (weighted) path problem is NP-hard on arbitrary graphs.
In the case of DAGs, however, it can be solved using standard algorithms relying on topological sorting in time $O(|V|+|E|)$~\cite{CLRS},
\textit{i.e.}, linear in the size of the graph.
Hence, the projection $\x$ can be determined in time $O(\dimension +|E|)$.

%
\section{Experiments}
\label{sec:experiments}

\subsection{Synthetic Data.}
We evaluate Alg.~\ref{algo:graphPM} and~\ref{algo:eps-net}
on synthetic data,
generated according to the model of Sec.~\ref{sec:data-model}.
We consider two metrics:
the loss function $\|\widehat{\x} \widehat{\x}^\top - \x_{\star} \x_{\star}\|_{\frob}$ and the Support Jaccard distance between the true signal $\x_{\star}$ and the estimate $\widehat{\x}$. 

For dimension~$\dimension$,
we generate a $(\dimension, \sparsity, d)$-layer graph~$G$,
with ${\sparsity =\log \dimension}$ 
and out-degree ${d = \nicefrac{\dimension}{\sparsity}}$, 
\textit{i.e.}, each vertex is connected to all vertices in the following layer.
We augment the graph with auxiliary source and terminal vertices $S$ and $T$ with edges to the original vertices as in Fig.~\ref{fig:layer-graph-ST-included}.

Per random realization, 
we first construct a signal $\x_\star \in \feasible$ as follows:
we randomly select an \st path $\pi$ in $G$,
and assign random zero-mean Gaussian values to the entries of $\x_{\star}$ indexed by $\pi$.
The signal is scaled to unit length.
Given $\x_{\star}$, we generate $\numsam$ independent samples  according to the spiked covariance model in~\eqref{gen-observations}.

Fig.~\ref{fig:synthetic-results} depicts 
the aforementioned distance metrics 
as a function of the number $\numsam$ of observations.
Results are the average of $100$ independent realizations.
We repeat the procedure for multiple values of the ambient dimension~$\dimension$.
\begin{figure}[tbhp!]   
   \centering
      \includegraphics[width=0.49\columnwidth]{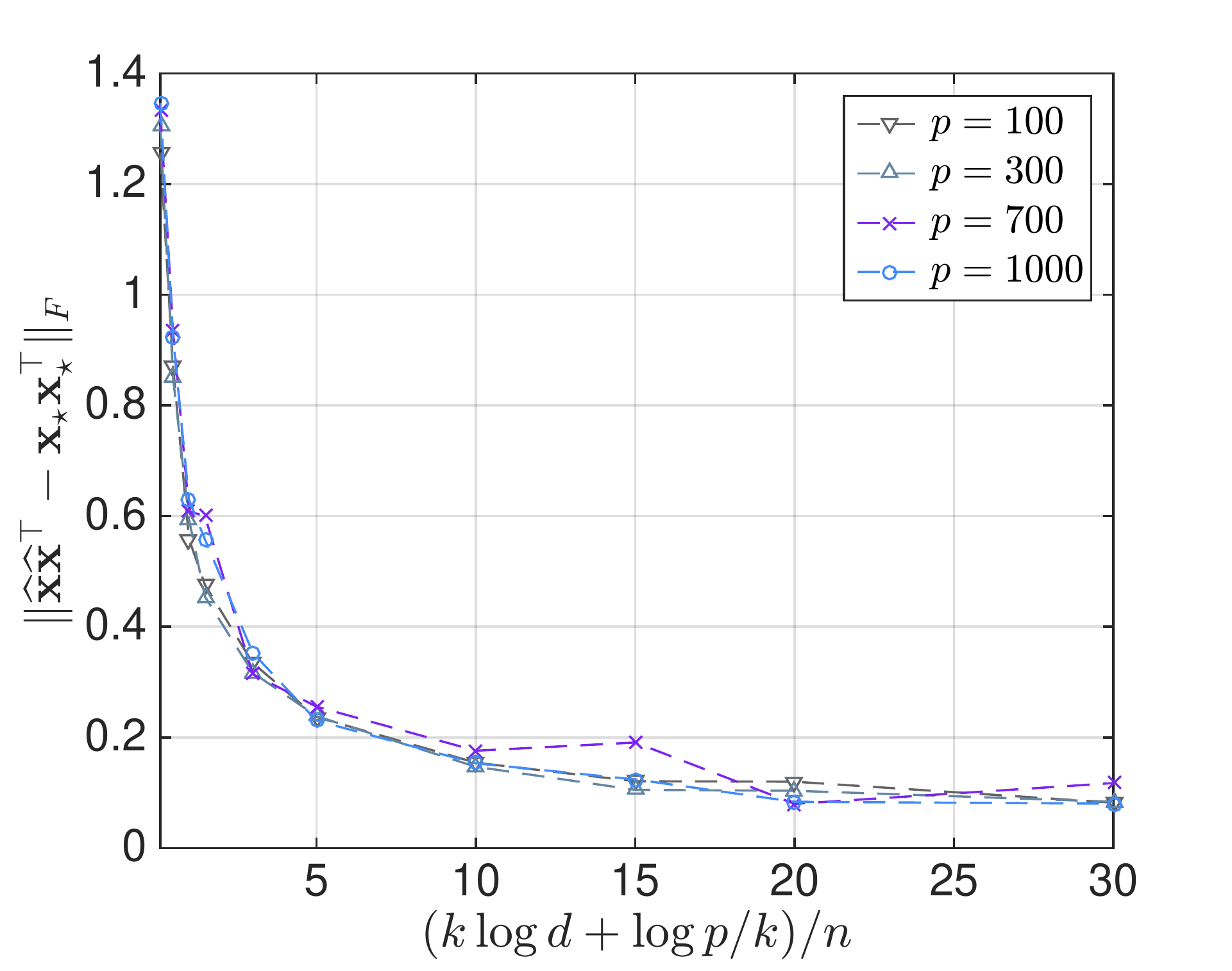}
   \includegraphics[width=0.49\columnwidth]{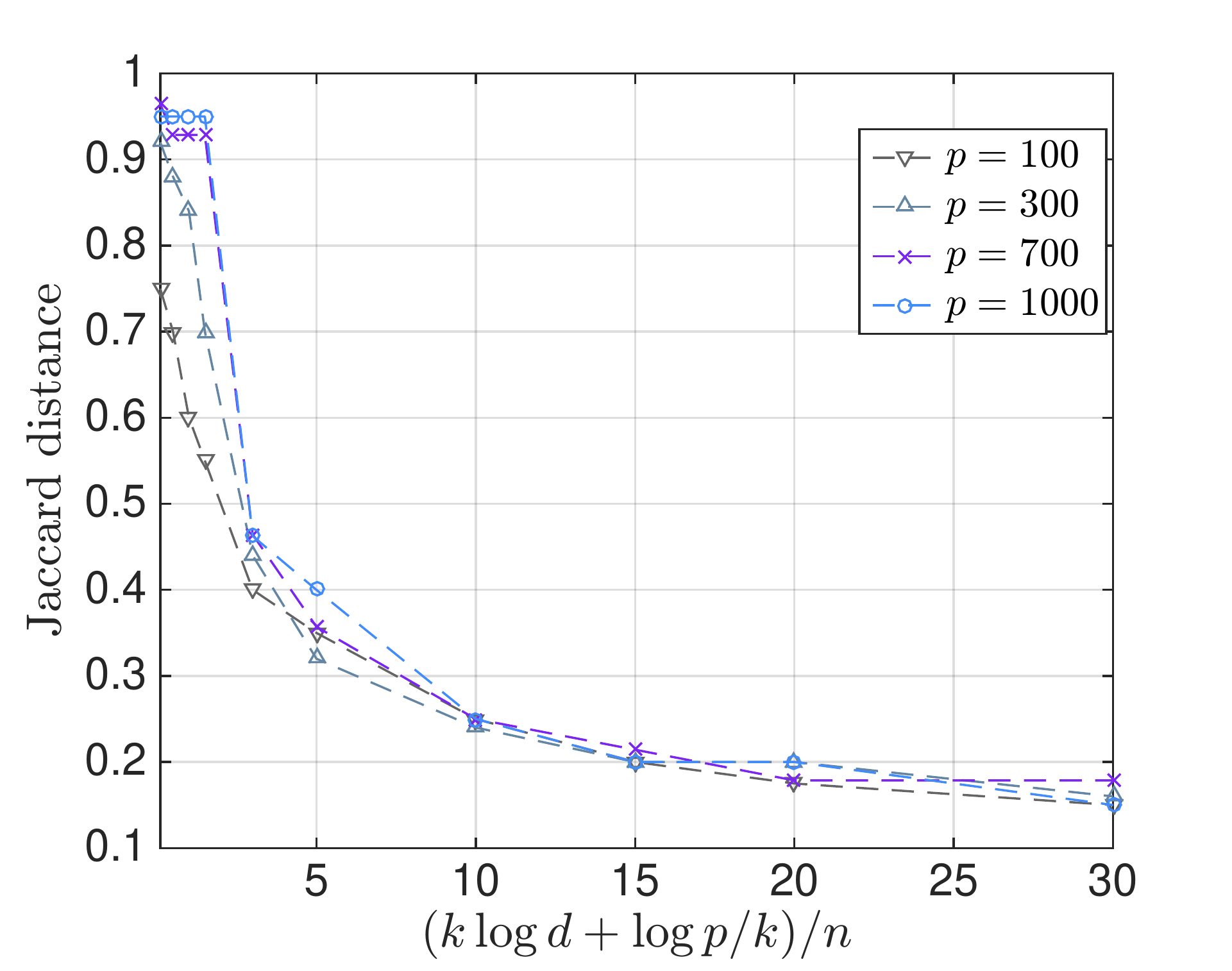}
   \caption{
   Metrics on the estimate $\widehat{\x}$ produced by Alg.~\ref{algo:graphPM} (Alg.~\ref{algo:eps-net} is similar) 
   as a function of the sample number (average of $100$ realizations).
   Samples are generated according to the spiked covariance model
   with signal ${\x_{\star}  \in \feasible}$ for a ${(p, k, d)}$-layer graph~$G$.
  Here, $k=\log{p}$ and $d=\nicefrac{p}{k}$.
  %
 We repeat for multiple values of~$\dimension$. }
   \label{fig:synthetic-results}
\end{figure}


\textbf{Comparison with Sparse PCA.\;}
We compare the performance of 
Alg.~\ref{algo:graphPM} and Alg.~\ref{algo:eps-net}
with their sparse PCA counterparts:
the Truncated Power Method of~\cite{yuan2013truncated}
and the Spannogram Alg. of~\cite{papailiopoulos:icml2013},
respectively.

\begin{figure}[!th]
	\centering
	\includegraphics[width=.49\columnwidth]{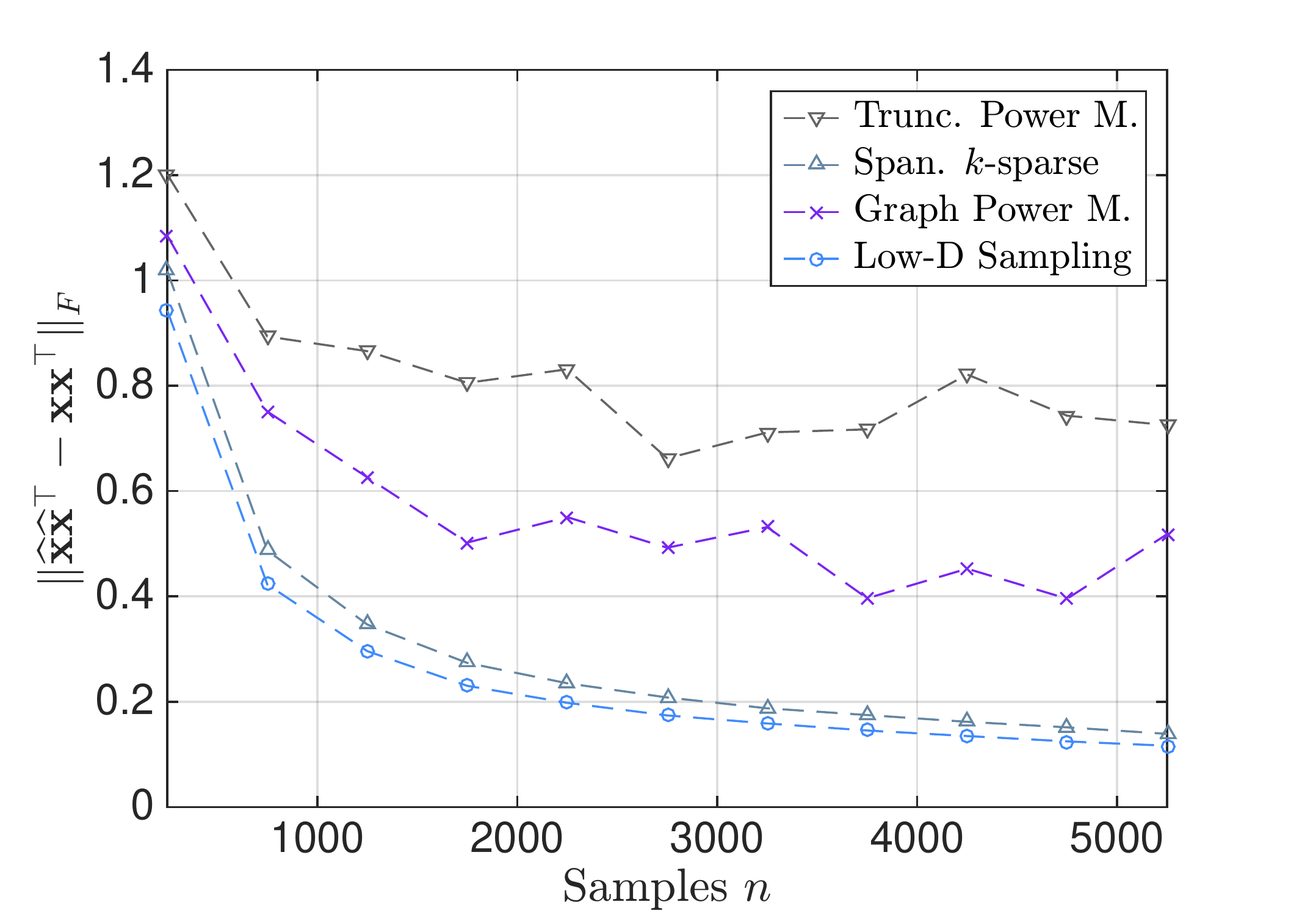}
	\includegraphics[width=.49\columnwidth]{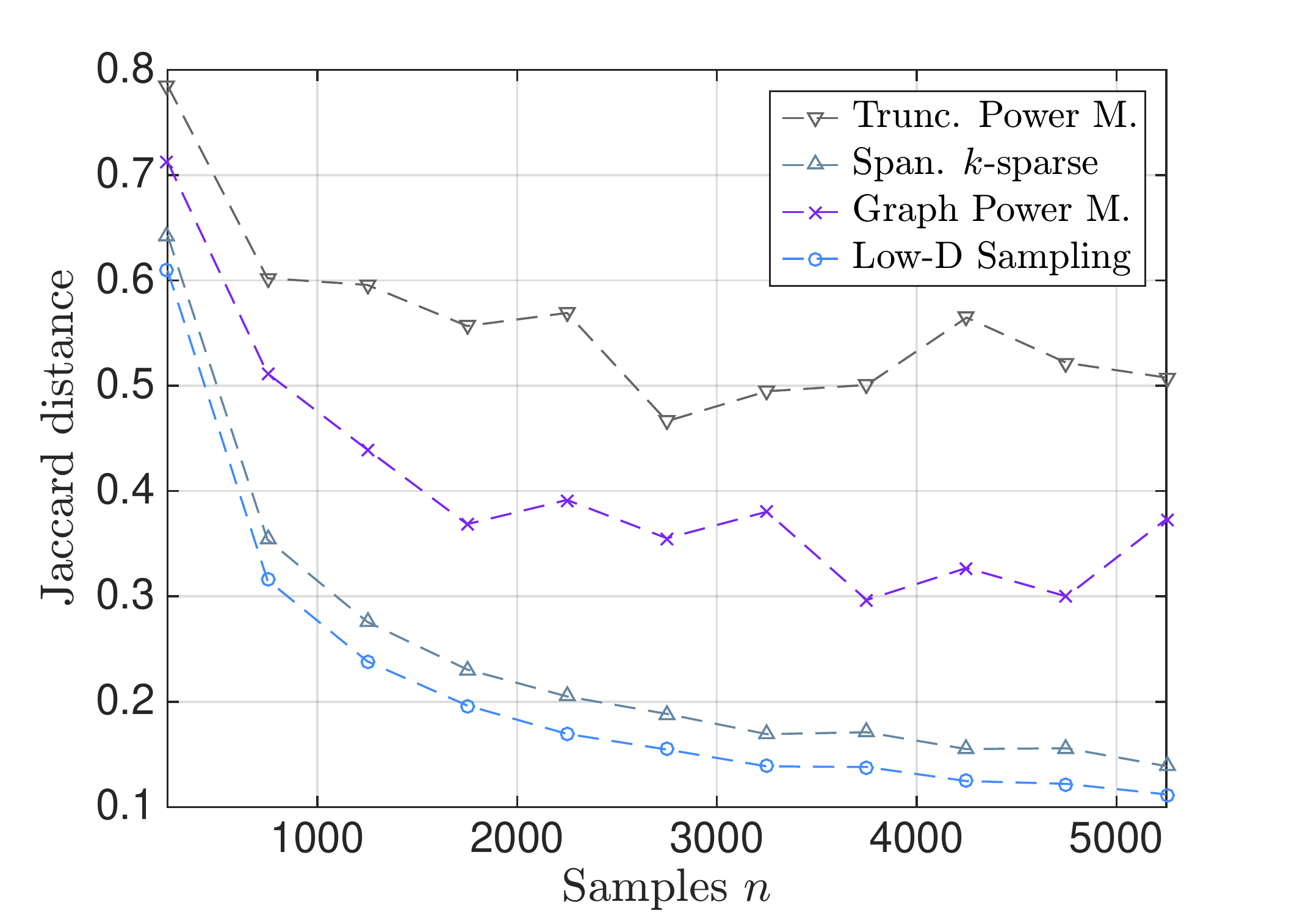}
	\caption{
	Estimation error between true signal $\x_{\star}$ and estimate $\widehat{\x}$ from $n$ samples.
	(average of $100$ realizations).
	Samples generated \textit{i.i.d.} $\sim N(\mathbf{0}, \covariance)$,
	where $\covariance$ has eigenvalues $\lambda_i = i^{-1/4}$ and  principal eigenvector $\x_{\star} \in \feasible$, for a $(p,k,d)$-layer graph $G$.
	(${\dimension = 10^{3}}$, ${\sparsity = 50}$, ${d = 10}$).	
	}
	\label{fig:comparison-results}
	\vspace{-1em}
\end{figure}

Fig.~\ref{fig:comparison-results}
depicts the metrics of interest as a function of the number of samples, for all four algorithms.
Here, samples are drawn \textit{i.i.d} from ${N(\mathbf{0}, \covariance)}$,
where $\boldsymbol{\Sigma}$
has principal eigenvector equal to $\x_{\star}$,
and power law spectral decay: $\lambda_i = i^{-1/4}$.
Results are an average of $100$ realizations.

The side information on the structure of $\x_{\star}$ assists the recovery:
both algorithms achieve improved performance compared to their sparse PCA counterparts.
Here, the power method based algorithms exhibit inferior performance,
which may be attributed to poor initialization. 
We note, though, that at least for the size of these experiments, 
the power method algorithms are significantly faster.


\subsection{Finance Data.}
This dataset contains daily closing prices for~$425$ stocks of the {S\&P 500} Index,
over a period of~$1259$ days ($5$-years): $02.01.2010$ -- $01.28.2015$,
collected from Yahoo! Finance.
Stocks are classified, according to the \emph{Global Industry Classification Standard}
(GICS), 
into $10$ business \emph{sectors} \textit{e.g.},
Energy, Health Care, Information Technology, etc (see Fig.~\ref{fig:finance-figure} for the complete list).

We seek a set of stocks comprising a single representative from each GICS sector,
which captures most of the variance in the dataset.
Equivalently, we want to compute a structured principal component constrained to have exactly $10$ nonzero entries; one for each GICS sector.

Consider a layer graph ${G=(V,E)}$ (similar to the one depicted in Fig.~\ref{fig:layer-graph-ST-included})
on ${p=425}$ vertices corresponding to the $425$ stocks,
partitioned into ${k=10}$ groups (layers) $\mathcal{L}_{1}, \hdots, \mathcal{L}_{10} \subseteq V$, corresponding to the GICS sectors.
Each vertex in layer $\mathcal{L}_{i}$ has outgoing edges towards all (and only the) vertices in layer $\mathcal{L}_{i+1}$. 
Note that (unlike Fig.~\ref{fig:layer-graph-ST-included})
layers do \emph{not} have equal sizes,
and the vertex out-degree varies across layers.
Finally, we introduce auxiliary vertices $S$ and $T$ connected with the original graph as in Fig.~\ref{fig:layer-graph-ST-included}.

Observe that any set of sector-representatives corresponds to an \st path in $G$, and vice versa.
Hence, the desired set of stocks can be obtained by finding a \emph{structured} principal component constrained to be supported along an \st path in $G$.
Note that the order of layers in $G$ is irrelevant.

Fig.~\ref{fig:finance-figure} depicts the subset of stocks selected by the proposed structure PCA algorithms (Alg.~\ref{algo:graphPM},~\ref{algo:eps-net}).
A single representative is selected from each sector.
For comparison, 
we also run two corresponding algorithms for sparse PCA,
with sparsity parameter ${k=10}$, equal to the number of sectors.
As expected, the latter yield components achieving higher values of
explained variance, but the selected stocks originate from only $5$ out of the $10$ sectors. 


%

\begin{figure}[!ht]
\centering
   \includegraphics[width=0.9\linewidth]{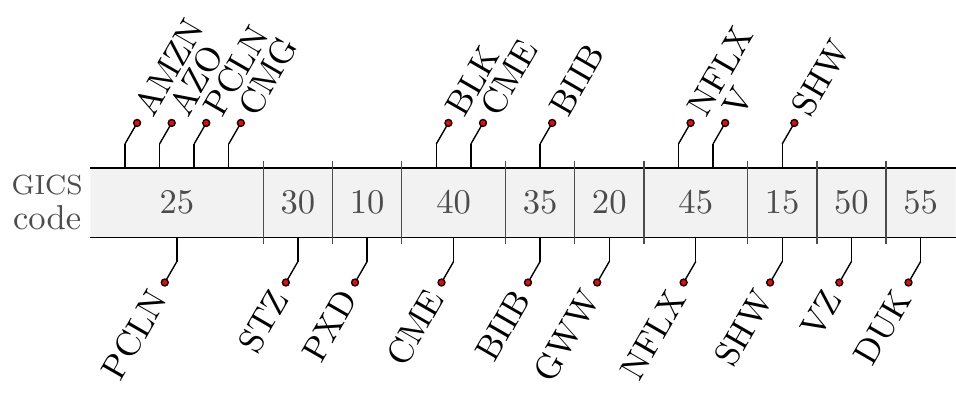}
   \begin{scriptsize}
   \begin{tabular}{l@{\hspace{.2em}}|@{\hspace{.2em}}l}
    \multicolumn{2}{l}{\bf GICS Sector Codes \cleaders\hbox{\rule[0.5ex]{2pt}{1.5pt}}\hfill\kern0pt}\\
   \begin{tabular}{c@{\hspace{1em}}l}
   $10$ & Energy \\
   $15$ & Materials \\
   $20$ & Industrials \\
   $25$ & Consumer Discretionary \\
   $30$ & Consumer Staples 
   \end{tabular}
   &
   \begin{tabular}{c@{\hspace{1em}}l}
   $35$ & Health Care \\
   $40$ & Financials \\
   $45$ & Information Tech. \\
   $50$ & Telecom. Services \\
   $55$ & Utilities  
   \end{tabular}
   \end{tabular}
   \end{scriptsize}
\caption{
The figure depicts the sets of $10$ stocks extracted by sparse PCA and our structure PCA approach.
Sparse PCA ($k=10$), selects $10$ stocks from $5$ GICS sectors (above).
On the contrary, our structured PCA algorithms yield a set of $10$ stocks containing a representative from each sector (below) as desired.
}
\label{fig:finance-figure}
\end{figure}

\subsection{Neuroscience Data.}
%

We use a single-session$/$single-participant resting state functional magnetic resonance imaging (resting state fMRI) dataset. The participant was not instructed to perform any explicit cognitive task throughout the scan~\cite{van2013wu}. 
Data was provided by the Human Connectome Project, WU-Minn Consortium.\footnote{%
   (Principal Investigators: David Van Essen and Kamil Ugurbil; 1U54MH091657) funded by the 16 NIH Institutes and Centers that support the NIH Blueprint for Neuroscience Research; and by the McDonnell Center for Systems Neuroscience at Washington University.
}

Mean timeseries of $\numsam = 1200$ points for $\dimension = 111$ regions-of-interest (ROIs) are extracted based on the Harvard-Oxford Atlas~\cite{desikan2006automated}. The timescale of analysis is restricted to $0.01$--\mbox{$0.1$Hz}. Based on recent results on resting state fMRI neural networks, 
we set the posterior cingulate cortex as a source node $S$, and the prefrontal cortex as a target node $T$~\cite{greicius2009resting}. 
Starting from $S$, we construct a layered graph with $\sparsity = 4$, based on the physical (Euclidean) distances between the center of mass of the ROIs: \textit{i.e.,} given layer $\mathcal{L}_i$, we construct $\mathcal{L}_{i+1}$ from non-selected nodes that are close in the Euclidean sense. Here,  $|\mathcal{L}_1| = 34 $ and $|\mathcal{L}_i| = 25$ for $i = 2,3,4$. Each layer is fully connected with its previous one.
No further assumptions are derived from neurobiology.

The extracted component suggests a directed pathway from the posterior cingulate cortex ($S$)
to the prefrontal cortex ($T$),
through the hippocampus ($1$), nucleus accumbens ($2$), parahippocampal gyrus ($3$), and frontal operculum ($4$) (Fig.~\ref{fig:brain}).
Hippocampus and the parahippocampal gyrus are critical in memory encoding, and 
have been found to be structurally connected to the posterior cingulate cortex and the prefrontal cortex~\cite{greicius2009resting}.
The nucleus accumbens receives input from the hippocampus, and plays an important role in memory consolidation~\cite{wittmann2005reward}. 
It is noteworthy that our approach has pinpointed the core neural components of the memory network, given minimal information.
\begin{figure}[tbh!]
\centering
   \includegraphics[width=0.55\linewidth]{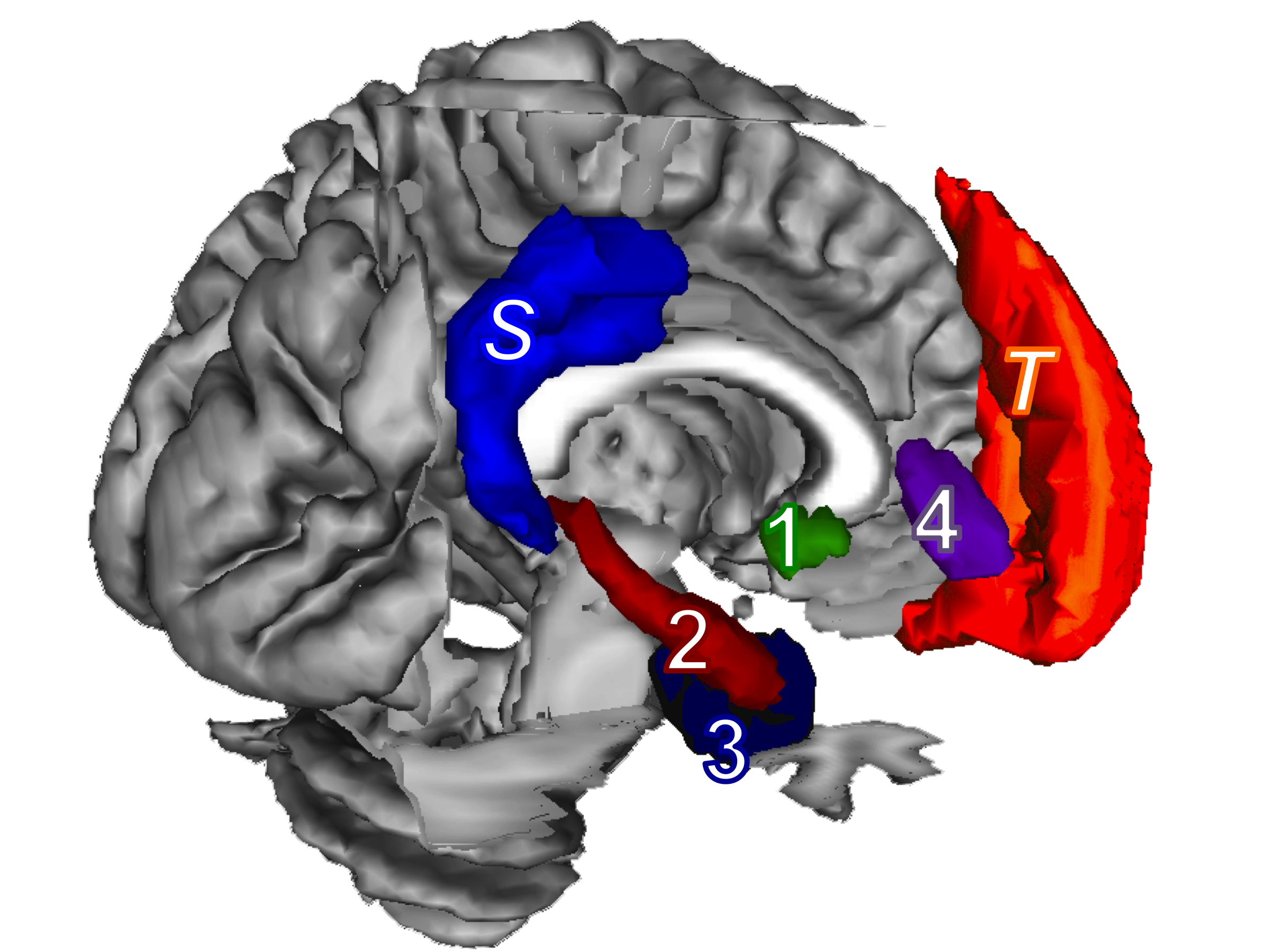}
   \includegraphics[width=0.9\linewidth]{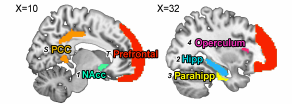}
   \caption{We highlight the nodes extracted for the neuroscience example.
   Source node set to the posterior cingulate cortex (S: PCC), and target to the prefrontal cortex (T: Prefrontal). The directed path proceeded from the nucleus accumbens (1: NAcc), hippocampus (2: Hipp), parahippocampal gyrus (3: Parahipp), and to the frontal operculum (4: Operculum). Here, $X$ coordinates (in mm) denote how far from the midline the cuts are.}
   \label{fig:brain}
\end{figure}

%

\section{Conclusions}
We introduced a new problem: 
sparse PCA where the set of feasible support sets is determined by a graph on the variables. 
We focused on the special case where feasible sparsity patterns coincide with paths on the underlying graph.
We provided an upper bound on the statistical complexity of the constrained quadratic maximization estimator~\eqref{eq:general-problem}, under a simple graph model,
complemented with a lower bound on the minimax error. 
Finally, we proposed two algorithms to extract a component 
accommodating the graph constraints and 
applied them on real data from finance and neuroscience.

A potential future direction is to
expand the set of graph-induced sparsity patterns (beyond paths)
that can lead to interpretable solutions 
and are computationally tractable.
We hope this work triggers future efforts to introduce and exploit such underlying structure in diverse research fields.

\newpage
\clearpage

\section{Acknowledgments}
The authors would like to acknowledge support from grants: NSF CCF 1422549, 1344364, 1344179 and an ARO YIP award.

\begin{small}
\bibliographystyle{icml2015}

\end{small}

\clearpage
\newpage
%
\section{Proof of Lemma~\ref{lem:local-packing:st-included} -- Local Packing Set}
\label{sec:proof-local-packing:st-included}
Towards the proof of Lemma~\ref{lem:local-packing:st-included},
we develop a modified version of the Varshamov-Gilbert Lemma
adapted to our specific model:
the set of characteristic vectors of the \st paths of a $(p,k,d)$-layer graph $G$.

Let $\delta_{H}(\x, \y)$ denote the Hamming distance between two points~$\x, \y \in \{0, 1\}^\dimension$:
\begin{align}
   \delta_{H}(\x, \y) 
   \;\eqdef\;
   \left|\{i:~ x_i \neq y_i\} \right|.
   \nonumber
\end{align}

\begin{lemma}
\label{lem:varshamov-st-included}
Consider a $(\dimension, \sparsity, d)$-layer graph~$G$ on $\dimension$ vertices and the collection $\P(G)$ of \st paths in $G$.
Let
\begin{align}
   \Omega \;\eqdef\;
	  \bigl\lbrace
		 \x \in \{0, 1\}^\dimension:
		 ~\supp(\x) \in \P(G)
	  \bigr\rbrace,
\nonumber
\end{align}
\textit{i.e.}, the set of characteristic vectors of all \st paths in~$G$.
For every $\xi \in (0,1)$, there exists a set, $\Omega_{\xi} \subset \Omega$ such that
\begin{align}
   \delta_{H}(\x, \y) > 
   2(1 - \xi)\cdot {\sparsity}, 
   \quad\forall \x, \y \in \Omega_{\xi}, \x \neq \y,
   \label{ham-dist-lb-st-in}
\end{align} 
and
\begin{align}
   \log \left|\Omega_{\xi} \right| 
   \geq \log\tfrac{p-2}{k} + (\xi\cdot {\sparsity} -1 ) \cdot \log d - k \cdot H(\xi),
   \label{eq:03}
\end{align}
where $H(\cdot)$ is the binary entropy function.
\end{lemma}
\begin{proof}
Consider a labeling~$1, \hdots, \dimension$ of the $\dimension$ vertices in~$G$,
such that variable $\omega_{i}$ is associated with vertex~$i$.
Each point $\boldsymbol{\omega} \in \Omega$ is the characteristic vector of a set in~$\P(G)$;
nonzero entries of $\boldsymbol{\omega}$ correspond to vertices along an \st path in~$G$.
With a slight abuse of notation, we refer to~$\boldsymbol{\omega}$ as a path in~$G$.
Due to the structure of the $(p,k,d)$-layer graph~$G$,
all points in $\Omega$ have exactly $k+2$ nonzero entries,
\textit{i.e.}, 
\begin{align}
   \delta_{H}\mathopen{}\left(\boldsymbol{\omega}, \mathbf{0}\right) = k+2, 
   \qquad \forall \boldsymbol{\omega} \in \Omega.
   \nonumber
\end{align}
Each vertex in $\boldsymbol{\omega}$ lies in a distinct layer of~$G$.
In turn, for any pair of points ${\boldsymbol{\omega}, \boldsymbol{\omega}^{\prime} \in \Omega}$,
\begin{align}
   \delta_{H}\mathopen{}\left(
	  \boldsymbol{\omega}, \boldsymbol{\omega}^{\prime}
	  \right) 
   = 
   2\cdot 
   \bigl( 
	  \sparsity -
		 \left|\left\lbrace 
			   i:~\omega_i = \omega^{\prime}_i = 1
		 \right\rbrace\right| - 2
	\bigr).
	\label{hamdist-alt-STin}
\end{align}
Note that the Hamming distance between the two points is a linear function of the number of their common nonzero entries,
while it can take only even values with a maximum value of~$2k$.

Without loss of generality, 
let~$S$ and~$T$ corresponding to vertices~$1$ and~$\dimension$, respectively.
Then, the above imply that
$$
{\omega_{1}=\omega_{\dimension}=1},
\quad \forall \boldsymbol{\omega} \in \Omega.
$$

Consider a fixed point ${\widehat{\boldsymbol{\omega}} \in {\Omega}}$,
and let $\B(\widehat{\boldsymbol{\omega}}, r)$ denote the Hamming ball of radius $r$ centered at $\widehat{\boldsymbol{\omega}}$, \textit{i.e.},
\begin{align}
 \B(\widehat{\boldsymbol{\omega}}, r)
 \;\;\eqdef\;\;
 \bigl\lbrace 
   \boldsymbol{\omega} \in \lbrace 0, 1\rbrace^{\dimension} :~\delta_{H}(\widehat{\boldsymbol{\omega}},\boldsymbol{\omega}) \le r
 \bigr\rbrace.
 \nonumber
\end{align}
The intersection $\B(\widehat{\boldsymbol{\omega}}, r) \cap {\Omega}$
corresponds to \st paths in~$G$ that have at least $k-\nicefrac{r}{2}$ additional vertices in common with~$\widehat{\boldsymbol{\omega}}$ besides  vertices~$1$ and~$\dimension$ that are common to all paths in $\Omega$:
\begin{align}
   &\B(\widehat{\boldsymbol{\omega}}, r) \cap {\Omega}
   \nonumber\\
   &=\left\{ {\boldsymbol{\omega} \in {\Omega}} : \delta_{H}(\widehat{\boldsymbol{\omega}}, \boldsymbol{\omega}) 
		 \leq r  \right\}
	  \nonumber \\
   &=\left\lbrace {\boldsymbol{\omega} \in {\Omega}}: \left|\{ i:~{\widehat{\omega}_i = \omega_i = 1} \}\right| \ge \sparsity-\tfrac{r}{2}+2
			 \right\rbrace
	 ,
	\nonumber 
\end{align}
where the last equality is due to~\eqref{hamdist-alt-STin}.
In fact, due to the structure of~$G$,
the set 
$\B(\widehat{\boldsymbol{\omega}}, r) \cap {\Omega}$
corresponds to the \st paths
that \emph{meet}~$\widehat{\boldsymbol{\omega}}$ 
in at least $k-\nicefrac{r}{2}$ intermediate layers.
Taking into account that 
$|\Gamma_{\text{in}}(v)|=|\Gamma_{\text{out}}(v)|=d$, for all vertices $v$ in $V(G)$ (except those in the first and last layer),
\begin{align}
	\left|\B(\widehat{\boldsymbol{\omega}}, r) \cap {\Omega}\right| 
	& \le
	\binom{k}{k-\tfrac{r}{2}} \cdot d^{k-\left(k-\tfrac{r}{2}\right)} 
	=
	\binom{k}{k-\tfrac{r}{2}} \cdot d^{\tfrac{r}{2}}.
	\nonumber
\end{align}

Now, consider a \emph{maximal} set $\Omega_{\xi} \subset {\Omega}$
satisfying~\eqref{ham-dist-lb-st-in},
\textit{i.e.}, a set that cannot be augmented by any other point in ${\Omega}$.
The union of balls $\B\mathopen{}\left(\boldsymbol{\omega}, 2(1 - \xi)\cdot ({\sparsity}-1)\right)$ over all ${\boldsymbol{\omega} \in \Omega_{\xi}}$ covers~${\Omega}$.
To verify that, note that 
if there exists $\boldsymbol{\omega}^{\prime} \in {\Omega}\backslash\Omega_{\xi}$ such that 
${\delta_{H}(\boldsymbol{\omega}, \boldsymbol{\omega}^{\prime}) > 2(1-\xi)\cdot ({\sparsity}-1)}$, 
${\forall \boldsymbol{\omega} \in \Omega_{\xi}}$,
then $\Omega_{\xi} \cup \lbrace\boldsymbol{\omega}^{\prime}\rbrace$ 
satisfies~\eqref{ham-dist-lb-st-in} contradicting the maximality of~$\Omega_{\xi}$.
Based on the above, 
\begin{align*}
   | \Omega| 
   &\leq 
	  \sum_{\boldsymbol{\omega} \in \Omega_{\xi}} |\B(\boldsymbol{\omega}, 2(1 - \xi) \cdot {\sparsity}) \cap \Omega| \\ 
   &\leq  \sum_{\x \in \Omega_{\xi}} \mathsmaller{\binom{k}{k-(1-\xi)k}} \cdot d^{(1-\xi)\cdot\sparsity}\\
   &\leq  \sum_{\x \in \Omega_{\xi}} \mathsmaller{\binom{k}{\xi{k}}} \cdot d^{(1-\xi)\cdot\sparsity}\\
   &\le |\Omega_{\xi}| \cdot 2^{k \cdot H(\xi)} \cdot d^{(1-\xi)\cdot\sparsity}.
\end{align*}
Taking into account that 
\begin{align*}
   \left|\Omega\right| 
   = |\P(G)|
   = \frac{\dimension-2}{\sparsity}\cdot d^{\sparsity - 1},
\end{align*}
we conclude that
\begin{align}
   \frac{\dimension-2}{\sparsity}\cdot d^{\sparsity - 1}
   &\leq |\Omega_{\xi}| \cdot 2^{k \cdot H(\xi)} \cdot d^{(1 - \xi)\cdot\sparsity},
   \nonumber
\end{align}
from which the desired result follows.
\end{proof}

\begin{customlemma}{\ref{lem:local-packing:st-included}}
	\emph{(Local Packing)}
	Consider a $(\dimension, \sparsity, d)$-layer graph~$G$ on $\dimension$ vertices
	with $k \ge 4$ and $\log{d} \ge 4\cdot H\mathopen{}\left(\sfrac{3}{4}\right)$.
	For any $\epsilon \in (0,1]$,
	there exists a set 
	${\mathcal{X}_{\epsilon} \subset \feasible}$ such that
   \begin{align*}
	  \epsilon / \sqrt{2}
	  \;<\; 
	  \| \x_{i} - \x_{j} \|_{2}
	  \;\le\; 
	  \sqrt{2} \cdot \epsilon,
   \end{align*}
   for all ${\x_i, \x_j \in \mathcal{X}_{\epsilon}}$, ${\x_i \neq \x_j}$,
   and
   \begin{align*}
	  \log |\mathcal{X}_{\epsilon}| 
	  \;\ge\; 
	  \log\frac{p-2}{k} + \frac{1}{4}\cdot {\sparsity} \log{d}.
   \end{align*} 
\end{customlemma}
\begin{proof}
Without loss of generality, 
consider a labeling $1,\hdots, \dimension$ of the $\dimension$ vertices in~$G$,
such that~$S$ and~$T$ correspond to vertices~$1$ and~$\dimension$, respectively.
Let
\begin{align}
   \Omega \;\eqdef\;
	  \bigl\lbrace
		 \x \in \{0, 1\}^\dimension:
		 ~\supp(\x) \in \P(G)
	  \bigr\rbrace,
\nonumber
\end{align}
where $\P(G)$ is the collection of \st paths in~$G$.
By Lemma~\ref{lem:varshamov-st-included},
and for ${\xi = \nicefrac{3}{4}}$,
there exists a set ${\Omega_{\xi} \subseteq \Omega}$ 
such that 
\begin{align}
   \delta_{H}(\boldsymbol{\omega}_{i}, \boldsymbol{\omega}_{j}) > 
   \frac{1}{2}\cdot {\sparsity},
   \label{dist-for-spec-xi}
\end{align} 
$\forall \boldsymbol{\omega}_{i}, \boldsymbol{\omega}_{j} \in \Omega_{\xi}$, $\boldsymbol{\omega}_{i} \neq \boldsymbol{\omega}_{j}$,
and,
\begin{align}
   \log \left|\Omega_{\xi} \right| 
   &\geq 
   \log\tfrac{p-2}{k} + \left(\tfrac{3}{4}\cdot {\sparsity} -1 \right) \log d - \sparsity \cdot H\mathopen{}\left(\tfrac{3}{4}\right) \nonumber \\
   &\geq 
   \log\tfrac{p-2}{k} + \tfrac{2}{4}\cdot {\sparsity} \cdot \log{d} - \sparsity \cdot H\mathopen{}\left(\tfrac{3}{4}\right) \nonumber \\
   &\geq 
   \log\tfrac{p-2}{k} + \tfrac{1}{4}\cdot {\sparsity} \cdot \log{d}
   \label{omegaksi-cardinality}
\end{align}
where the second and third inequalites hold under the assumptions of the lemma; $k \ge 4$ and $\log{d} \ge 4\cdot H(\nicefrac{3}{4})$.

Consider the bijective mapping $\psi : \Omega_{\xi} \rightarrow \R^{\dimension}$ defined as
\begin{align*}
	\psi (\boldsymbol{\omega}) 
	= \left[ \sqrt{\tfrac{(1 - \epsilon^2)}{2}} \cdot \omega_{1},\; 
		     \tfrac{\epsilon}{\sqrt{\sparsity}} \cdot \boldsymbol{\omega}_{2:\dimension-1},\;
		     \sqrt{\tfrac{(1 - \epsilon^2)}{2}} \cdot \omega_{p}
	   \right].
\end{align*} 
We show that the set
\begin{align}
	\mathcal{X}_{\epsilon}
	\;\eqdef\;
	\left\lbrace
	  \psi(\boldsymbol{\omega}):~ \boldsymbol{\omega} \in \Omega_{\xi} 
	\right\rbrace.
	\nonumber
\end{align}
has the desired properties.
First, to verify that $\mathcal{X}_{\epsilon}$ is a subset of $\mathcal{X}(G)$,
note that
${\forall \boldsymbol{\boldsymbol{\omega}} \in \Omega_{\xi} \subset \Omega}$,
\begin{align}
 \supp\mathopen{}\left(\psi(\boldsymbol{\omega})\right) = 
\supp(\boldsymbol{\omega}) \in \P(G),
\end{align}
and
\begin{align}
   \|\psi(\boldsymbol{\omega})\|_2^{2}
   = 
	  2 \cdot \frac{{(1 - \epsilon^2)}}{2} 
	  + 
	  \frac{\epsilon^2}{\sparsity} \cdot \sum_{i = 2}^{\dimension-1} \omega_{i}
   = 
   1. \nonumber
\end{align}

Second, for all pairs of points ${\x_i, \x_j \in \mathcal{X}_{\epsilon}}$,
\begin{align}
   \|\x_i - \x_j\|_2^{2} 
   =
   \delta_{H}(\boldsymbol{\omega}_{i}, \boldsymbol{\omega}_{j}) \cdot \frac{\epsilon^{2}}{k} 
   \leq 
   2 \cdot k \cdot \frac{\epsilon^{2}}{k}
   = 
   2 \cdot \epsilon^{2}. \nonumber
\end{align}
The inequality 
follows from the fact that 
$\delta_{H}\mathopen{}\left(\boldsymbol{\omega}, \mathbf{0}\right)=k+2$ $\omega_{1}=1$ and $\omega_{\dimension}=1$, $\forall \boldsymbol{\omega} \in \Omega_{\xi}$,
and in turn
\begin{align}
   \delta_{H}(\boldsymbol{\omega}_{i}, \boldsymbol{\omega}_{j}) 
   \leq 
   2 \cdot \sparsity.
   \nonumber
\end{align}
Similarly, for all pairs
${\x_i, \x_j \in \mathcal{X}_{\epsilon}}$,
$\x_i \neq \x_j$,
\begin{align}
   \|\x_i - \x_j\|_{2}
   =
   \delta_{H}(\boldsymbol{\omega}_{i}, \boldsymbol{\omega}_{j}) \cdot \frac{\epsilon^{2}}{k} 
   \ge
   \frac{1}{2} \cdot k \cdot \frac{\epsilon^{2}}{k}
   = 
   \frac{\epsilon^{2}}{2}, \nonumber
\end{align}
where the inequality is due to \eqref{dist-for-spec-xi}.
Finally, 
the lower bound on the cardinality of $\mathcal{X}_{\epsilon}$
follows immediately from~\eqref{omegaksi-cardinality}
and the fact that 
$|\mathcal{X}_{\epsilon}| = |\Omega_{\xi}|$,
which completes the proof.
\end{proof} 


\section{Details in proof of Lemma~\ref{thm:lower-bound-st-included}}
\label{sec:lb-proof-details}
We want to show that if
\begin{align}
   \epsilon^2 
   = 
   \min\mathopen{} 
	  \left\lbrace 
			1, 
			~ \frac{C^{\prime} \cdot (1 + \beta)}{\beta^2} 
			\cdot \frac{\log\tfrac{p-2}{k} + \tfrac{\sparsity}{4}\cdot \log{d}}{\numsam}
	  \right\rbrace,
   \nonumber
\end{align}
for an appropriate choice of $C^{\prime} > 0$,
then the following two conditions (Eq.~\eqref{conditions-for-constant-eps-st-included}) are satisfied:
\begin{align}
	\numsam 
	\cdot \frac{2 \epsilon^2 \beta^2 }{(1+\beta)}
	\frac{1}{\log |\mathcal{X}_{\epsilon}|} \leq \frac{1}{4} 
	\;\;\text{and}\;\;
	\log |\mathcal{X}_{\epsilon}| \geq 4 \log 2.
	\nonumber
\end{align}

For the second inequality, 
recall that by Lemma~\ref{lem:local-packing:st-included},
\begin{align}
   \log |\mathcal{X}_{\epsilon}| 
   \,\geq\,
   \log\frac{p-2}{k} + \frac{1}{4}\cdot {\sparsity} \log{d} 
   \;>\;0.
   \label{part-2b-st-included}
\end{align}
Under the assumptions of Thm.~\ref{thm:lower-bound-st-included} on the parameters~$\sparsity$ and~$d$ (note that $p-2 \ge k \cdot d$ by the structure of~$G$),
$$
   \log |\mathcal{X}_{\epsilon}| 
   \ge
   \log\frac{p-2}{k} + \frac{\sparsity}{4}\cdot\log{d} 
   \ge
   4 \cdot H(\sfrac{3}{4})
   \ge
   4 \log 2,
$$ 
which is the desired result.

For the first inequality, we consider two cases:
\begin{enumerate}[label=\textbullet, leftmargin=*, labelindent=0em, align=left, labelsep=.5em, itemindent=0em, labelwidth=1em, itemsep=.2em, topsep=0pt, after=\vspace{.0em}, parsep=0pt]
 \item 
 First, we consider the case where ${\epsilon^2 = 1}$,
\textit{i.e.}, 
\begin{align}
   \epsilon^{2} 
   = 1 
   \leq 
   \frac{C^{\prime} \cdot (1 + \beta)}{\beta^2} 
			\cdot \frac{\log\tfrac{p-2}{k} + \tfrac{\sparsity}{4}\cdot \log{d}}{\numsam}.
   \nonumber
\end{align}
Equivalently,
\begin{align}
   \numsam \cdot \frac{2 \epsilon^2 \beta^2 }{(1+\beta)}
   \leq
   2 \cdot C^{\prime} \cdot
   \left(\log\frac{p-2}{k} + \frac{\sparsity}{4}\cdot \log{d}\right).
   \label{part-1-case-1-st-included}
\end{align}

\item
In the second case, 
\begin{align}
   \epsilon^2 
   =  
   \frac{C^{\prime} \cdot (1 + \beta)}{\beta^2} 
			\cdot \frac{\log\tfrac{p-2}{k} + \tfrac{\sparsity}{4}\cdot \log{d}}{\numsam},
   \nonumber
\end{align}
which implies that
\begin{align}
   \numsam \cdot \frac{2 \epsilon^2 \beta^2 }{(1+\beta)}
   =
   2 \cdot C^{\prime} \cdot
   \left(\log\frac{p-2}{k} + \frac{\sparsity}{4}\cdot \log{d}\right).
   \label{part-1-case-2-st-included}
\end{align}
\end{enumerate}

Combining~\eqref{part-1-case-1-st-included} or~\eqref{part-1-case-2-st-included}, 
with~\eqref{part-2b-st-included},
we obtain
\begin{align}
   \numsam \cdot \frac{2 \epsilon^2 \beta^2 }{(1+\beta)}
	\frac{1}{\log |\mathcal{X}_{\epsilon}|}
   \,\leq\,
   2\cdot C^{\prime}
   \,\le \,
   \frac{1}{4}
   \nonumber
\end{align}
for ${C^{\prime} \le \nicefrac{1}{8}}$. 

We conclude that for $\epsilon$ chosen as in~\eqref{epsilon-range-st-included}, the conditions in~\eqref{conditions-for-constant-eps-st-included} hold.

\section{Other}
\begin{assumption}
   \label{ub-assumption}
 There exist \textit{i.i.d.} random vectors $\mathbf{z}_{1}, \hdots, \mathbf{z}_{n} \in \mathbb{R}^{p}$,
 such that $\mathbb{E}\mathbf{z}_{i} = \mathbf{0}$ and $\mathbb{E}{\mathbf{z}_{i}\mathbf{z}_{i}^{\transpose}} = \mathbb{I}_{\dimension}$,
 \begin{align}
  \y = \mu + \covariance^{\sfrac{1}{2}} \mathbf{z}_{i}
 \end{align}
 and
 \begin{align}
  \sup_{\x \in \mathbb{S}_{2}^{\dimension-1}} \|\mathbf{z}_{i}^{\transpose} \x \|_{\psi_{2}} \le K,
 \end{align}
where $\mu \in \mathbb{R}^{\dimension}$ and $K>0$ is a constant depending on the distribution of $\mathbf{z}_{i}$s.
\end{assumption}


\end{document}